\documentclass[twoside]{article}

\usepackage[accepted]{aistats2022}
\usepackage[round]{natbib}

\input{aa.sty}

\begin{document}

%
\runningtitle{Accelerated Gradient Flow: Risk, Stability, and Implicit Regularization}

%
\runningauthor{Yue Sheng, Alnur Ali}

\twocolumn[

\aistatstitle{Accelerated Gradient Flow:\\ Risk, Stability, and Implicit Regularization}

\aistatsauthor{Yue Sheng$^*$ \And Alnur Ali$^*$}

\aistatsaddress{University of Pennsylvania \And Stanford University} ]

\begin{abstract}
Acceleration and momentum are the de facto standard in modern applications of machine learning and optimization, yet the bulk of the work on implicit regularization focuses instead on unaccelerated methods.  In this paper, we study the statistical risk of the iterates generated by Nesterov's accelerated gradient method and Polyak's heavy ball method, when applied to least squares regression, drawing several connections to explicit penalization.  We carry out our analyses in continuous-time, allowing us to make sharper statements than in prior work, and revealing complex interactions between early stopping, stability, and the curvature of the loss function.
\end{abstract}

\section{INTRODUCTION}
Acceleration \citep{Nesterov83,Nesterov05b,Nesterov88,Nesterov07,Tseng08,BeckTe09} and momentum \citep{Polyak64,Polyak87} are enormously popular tools for convex and non-convex optimization alike, playing central roles in many modern applications of machine learning and statistics.  As is telling, a number of recent optimization algorithms commonly used to fit deep neural networks, e.g., Adam \citep{KingmaBa15}, AdaGrad \citep{DuchiHaSi11}, and RMSProp \citep{HintonSrSw12}, either leverage momentum directly, or are routinely modified in practice to incorporate it \citep{SutskeverMaDaHi13,YangLiLi16,Dozat16,WilsonRoStSrRe17,ZouShJiSuLi18,ZouShJiZhLi19,DefossezBoBaUs20}.

One plausible explanation for (at least some of) the surprising recent successes of deep neural networks is that these optimization algorithms perform \textit{implicit regularization}, i.e., the iterates generated by these algorithms possess a kind of statistical regularity even without the use of any explicit regularizer \citep{NacsonSrSo18,GunasekarLeSoSr18,SoudryHoNaGuSr18,SuggalaPrRa18,AliKoTi19,PoggioBaLi19}.  Implicit regularization has undoubtedly seen an explosion of interest over the last few years, but most of the analyses focus on \textit{unaccelerated} methods, e.g., the standard gradient descent iteration \citep{LeeSiJoRe16,GunasekarWoBhNeSr17,GunasekarLeSoSr18,NacsonSrSo18,JacotGaHo18,SoudryHoNaGuSr18,SuggalaPrRa18,PaglianaRo19,DuLeLiWaZh19,DuZhPoSi19,AliKoTi19,HastieMoRoTi19,AmariBaGrLiNiSuWuXu20,VaskeviciusKaRe20,BartlettMoRa21}, or (mini-batch) stochastic gradient descent \citep{NacsonSrSo18,GunasekarLeSoSr18,JainKaKiNeSi18,AliDoTi20,WuZoBrGu20}, despite the prevalance of accelerated methods in practice.  The reason for this focus is probably simplicity, as acceleration can be somewhat difficult to study formally.

In this paper, we exactly characterize the risk of both Nesterov's accelerated gradient method, as well as Polyak's heavy ball method, across the entire optimization path, when applied to the fundamental problem of least squares regression (i.e., without regularization); as a result, we draw a number of connections to ridge regression \citep{HoerlKe70}, i.e., to explicit penalization.  A key feature of our approach is that we carry out the analyses in continuous-time, simplifying many of the arguments and allowing us to make sharper statements than have been made in prior work.

\paragraph{Summary of Contributions.}  A summary of our contributions in this paper is as follows.
\begin{itemize}[itemsep=0.0ex]
  \item We derive exact expressions for the estimation risk of the iterates generated by Nesterov's accelerated gradient method and Polyak's heavy ball method, holding across the entire continuous-time optimization path, i.e., for any $t \geq 0$.  To do so, we study two second-order differential equations that we call \textit{accelerated gradient flow} and \textit{heavy ball flow}, respectively.

  \item We show that, under optimal tuning, the risk associated with Nesterov's method is at most 1.5991 times the risk of ridge regression.
  
  \item We demonstrate that, in general, it is impossible to give a tight coupling between Nesterov's method (or the heavy ball method) and ridge across the \textit{entire} path, due to instability, i.e., because the variance of accelerated methods can grow quickly and without bound, depending on the spectrum of the sample covariance matrix.  Moreover, because accelerated methods are not descent methods in general, their risks tend to oscillate.  As a whole, we provide a more refined picture of the ``stability-convergence speed trade-off'' than prior work, which is qualitatively different from that for gradient descent.

  \item Nonetheless, we give a tight bound on the \textit{relative} parameter error, i.e., the $\ell_2$ norm of the difference between the accelerated vs.~ridge coefficients, normalized by the length of the ridge coefficients, holding across the entire path.
  
  \item We provide numerical experiments supporting our general theory, showing that under idealized conditions accelerated gradient methods can indeed reach low-risk solutions faster than standard gradient methods can, but that early-stopping should be used with care in general.
\end{itemize}

\paragraph{Outline.}  Here is an outline for the rest of this paper.  In the next section, we give some background on acceleration and implicit regularization, and review related work.  In Section \ref{sec:nest} that follows, we introduce our continuous-time framework, and present our main results for Nesterov's method.  In Section \ref{sec:hb}, we present our results for the heavy ball method.  We give numerical evidence for our findings in Section \ref{sec:exps}, and wrap up with a short discussion in Section \ref{sec:disc}. 

\section{BACKGROUND}
\label{sec:bkgd}

\subsection{Least Squares, Ridge Regression, Gradient Descent, and Gradient Flow}
Given a fixed design matrix $X \in \R^{n \times p}$, and response points $y \in \R^n$ arising from a canonical linear model
\begin{equation}
  y = X \beta_0 + \varepsilon, \label{eq:data_model}
\end{equation}
with underlying coefficients $\beta_0 \sim (0, (r^2/p) I)$ and noise $\varepsilon \sim (0, \sigma^2 I)$, for some $\sigma > 0$, the usual least squares regression estimate is given by solving
\begin{equation}
  \minimize_{\beta \in \R^p} \; \frac{1}{2n} \| y - X \beta \|_2^2. \label{eq:ls}
\end{equation}

Applying the standard (discrete-time) gradient descent iteration to least squares regression in \eqref{eq:ls} gives
\begin{equation}
  \beta^{(k)} = \beta^{(k-1)} + \frac{\epsilon}{n} \cdot X^T (y - X \beta^{(k-1)}), \quad \beta^{(0)} = 0, \label{eq:gd}
\end{equation}
where $k \geq 0$ is an iteration counter, and $\epsilon > 0$ is a fixed step size.  Taking infinitesimally small steps in \eqref{eq:gd}, i.e., sending the step size $\epsilon \to 0$, yields the ordinary differential equation called gradient flow,
\begin{equation}
  \dot \beta(t) = X^T (y - X \beta(t)), \quad \beta(0) = 0, \quad t \geq 0. \label{eq:gf}
\end{equation}
Here, $\dot \beta(t)$ is the time derivative of $\beta : \R_+ \to \R^p$.

Also, recall that the ridge regression estimate \citep{HoerlKe70}, for any $\lambda \geq 0$, is simply
\begin{equation}
  \hat \beta^\ridge(\lambda) = \argmin_{\beta \in \R^p} \frac{1}{2n} \| y - X \beta \|_2^2 + \frac{\lambda}{2} \| \beta \|_2^2. \label{eq:ridge}
\end{equation}

\subsection{Acceleration and Momentum}
Now let $0 \leq s_1 \leq \cdots \leq s_p$ denote the singular values of the sample covariance matrix $\hat \Sigma = X^T X/n$, and write $\mu = s_1$ and $L = s_p$.  It is a standard fact that gradient descent, as in \eqref{eq:gd}, with step size $\epsilon \leq 1/L$ converges to a solution of \eqref{eq:ls} in $O(1/k)$ iterations, which is suboptimal for first-order methods.  On the other hand, Nesterov's accelerated gradient method attains the optimal $O(1/k^2)$ rate, and works as follows.  Nesterov's method composes a gradient step with a momentum adjustment, i.e.,
\begin{equation} \label{eq:nest}
  \begin{aligned}
    \beta^{(k)}_\Nest & = \beta^{(k-1)}_\Nest + \frac{\epsilon}{n} \cdot X^T (y - X \theta^{(k-1)}) \\
    \theta^{(k)} & = \beta^{(k)}_\Nest + \frac{k-1}{k+2} \cdot (\beta^{(k)}_\Nest - \beta^{(k-1)}_\Nest).
  \end{aligned}
\end{equation}

Nesterov's method \eqref{eq:nest} is itself a refinement of the heavy ball method, introduced by Polyak, which simply allows the previous iterations to carry some momentum, i.e., we perform the update
\begin{equation} \label{eq:hb}
  \begin{aligned}
    \beta^{(k)}_\hb & = \beta^{(k-1)}_\hb + \frac{\epsilon}{n} \cdot X^T (y - X \beta^{(k-1)}_\hb) \\
    & \hspace{1.05in} + \eta \cdot (\beta^{(k-1)}_\hb - \beta^{(k-2)}_\hb),
  \end{aligned}
\end{equation}
where $\eta \geq 0$ is the momentum parameter.  We initialize the iterations \eqref{eq:nest}, \eqref{eq:hb} at zero.

Unlike Nesterov's method \eqref{eq:nest}, which is (globally) convergent for smooth convex objectives, i.e., those having Lipschitz continuous gradients, the heavy ball method \eqref{eq:hb} need not converge even for strongly convex objectives, though it is globally convergent for \eqref{eq:ls}, provided that $\hat \Sigma$ is positive definite.  Some works have demonstrated instances of divergence \citep{LessardRePa16}, and others have shown that stronger conditions are in general required to establish convergence \citep{ZavrievKo93,GhadimiFeJo14,OchsChBrPo14}.

\subsection{Accelerated Continuous-Time Dynamics}
Polyak's method \eqref{eq:hb} was developed with a physical interpretation in mind, so it is natural to study the iteration \eqref{eq:hb} in continuous-time.  In fact, \eqref{eq:hb} can be seen as the discretization of the second-order ordinary differential equation
\begin{equation}
  \ddot \beta(t) + 2 \mu^{1/2} \dot \beta(t) = \frac{1}{n} X^T (y - X \beta(t)), \label{eq:hb_ode}
\end{equation}
with the initialization $\beta(0) = \dot \beta(0) = 0$, which we refer to as \textit{heavy ball flow}.

On the other hand, Nesterov's method \eqref{eq:nest} has remained a bit mysterious, despite its popularity, over the years.  We discuss a few different interpretations for the iteration \eqref{eq:nest} in the next section, but for our purposes the interpretation of \citet{SuBoCa14} turns out to be most useful.  \citet{SuBoCa14} view \eqref{eq:nest} as the discretization of another closely related differential equation, in analogy to what was done with \eqref{eq:hb}, \eqref{eq:hb_ode}.  We refer to this differential equation, i.e.,
\begin{equation}
  \ddot \beta(t) + \frac{3}{t} \dot \beta(t) = \frac{1}{n} X^T (y - X \beta(t)), \label{eq:nest_ode}
\end{equation}
with the initialization $\beta(0) = \dot \beta(0) = 0$, as \textit{accelerated gradient flow}.  In what follows, the differential equations \eqref{eq:hb_ode} and \eqref{eq:nest_ode} both play a key role in understanding the statistical properties of \eqref{eq:hb} and \eqref{eq:nest}.

\subsection{Related Work}
Before turning to our main results in this paper, we first give a brief survey of related work.

\paragraph{Accelerated methods.}  Momentum was introduced by Polyak \citep{Polyak64,Polyak87}.  Nesterov refined the idea, and furthermore showed that acceleration is optimal among first-order methods \citep{Nesterov83,Nesterov05b,Nesterov88,Nesterov07}.  In direct analogy to \citet{Polyak64,Polyak87}, and building on ideas found in \citet{ODonoghueCa15}, \citet{SuBoCa14} interpret the iteration \eqref{eq:nest} as the discretization of \eqref{eq:nest_ode}.  \citet{KricheneBaBa15,WibisonoWiJo16,WilsonReJo16,BetancourtJoWi18} follow up, and show that a broad class of first-order methods, including Nesterov's method, may be seen as discretizations of a ``master'' differential equation, called the Bregman Lagrangian.  \citet{ShiDuJoSu21} study a slightly different family of differential equations than \eqref{eq:hb_ode}, \eqref{eq:nest_ode}, which more accurately reflect the different underlying dynamics of Nesterov's method and the heavy ball method, when applied to strongly convex objectives.  \citet{AllenOr17} show that \eqref{eq:nest} may be seen as a kind of combination of gradient descent and mirror descent.  \citet{BubeckLeSi15} give a geometric interpretation of the iteration \eqref{eq:nest}, based on localization-type ideas.

\citet{WilsonRoStSrRe17,LevyDu19} carry out interesting empirical and theoretical analyses, characterizing conditions under which accelerated methods may outperform their adaptive gradient (i.e., variable metric) counterparts, and vice-versa.

Finally, and especially relevant to the current paper, \citet{PaglianaRo19} give excess error bounds for Nesterov's method, showing that it attains parametric rates of convergence under suitable smoothness conditions.  To prove these results, the authors interpret Nesterov's method as a spectral shrinkage map and consider a bias-variance decomposition, showing that the accelerated bias drops faster than the unaccelerated bias as a by-product and explaining the instability commonly seen with accelerated methods.  As we will see in Section \ref{sec:nest}, these techniques are reminiscent of our own.

Additionally, \citet{ChenJiYu18} show the sum of the worst-case optimization error and the algorithmic stability (in the sense of \citet{BousquetEl02,HardtReSi16}), associated with any iterative algorithm including Nesterov's method and the heavy ball method, over a chosen loss function class, is lower bounded by the minimax excess risk over that same class.  The authors claim this exhibits a so-called stability-convergence trade-off, i.e. a large minimax risk implies a slow convergence rate (and conversely), again shedding light on instability.

\paragraph{Implicit regularization.}  Nearly all of the work on implicit regularization so far has looked at unaccelerated first-order methods, as mentioned in the introduction.  The literature here is massive, so we cannot give a complete coverage, but some key references include \citet{LeeSiJoRe16,GunasekarWoBhNeSr17,GunasekarLeSoSr18,NacsonSrSo18,JacotGaHo18,SoudryHoNaGuSr18,SuggalaPrRa18,PaglianaRo19,DuLeLiWaZh19,DuZhPoSi19,AliKoTi19,AliDoTi20,WuDoReWuLiGuWaLi20,AmariBaGrLiNiSuWuXu20,VaskeviciusKaRe20,WuZoBrGu20}.

In particular, we point out \citet{SuggalaPrRa18,DuZhPoSi19,AliKoTi19,AliDoTi20,WuDoReWuLiGuWaLi20} as very relevant works, performing their respective analyses in continuous-time, which is similar to the spirit of the current paper.

\section{NESTEROV ACCELERATION}
\label{sec:nest}

We begin by considering Nesterov's method.

\subsection{An Exact Risk Expression}
Fix $\beta_0 \in \R^p$.  For any estimator $\hat \beta \in \R^p$, we write
\[
  \Risk(\hat \beta;\beta_0) = \E \| \hat \beta - \beta_0 \|_2^2.
\]
Denoting any solution to accelerated gradient flow \eqref{eq:nest_ode} as $\hat \beta^\Nest(t)$, for $t \geq 0$, our goal is to characterize $\Risk(\hat \beta^\Nest(t);\beta_0)$, and relate it to $\Risk(\beta^\ridge(\lambda);\beta_0)$, under an appropriate correspondence between the tuning parameters $t$ and $\lambda$.

A natural strategy, pursued by prior work for gradient descent \citep{SuggalaPrRa18,AliKoTi19,AliDoTi20}, is to write down a closed-form expression for $\hat \beta^\Nest(t)$, and hope that it is amenable to analysis.  Unfortunately, from looking back at \eqref{eq:nest_ode}, this is difficult.  Therefore, we proceed slightly indirectly.

Write $X = n^{1/2} U S^{1/2} V^T$ for a singular value decomposition, so that $\hat \Sigma = V S V^T$ with $S \in \R^{p \times p}$ containing $s_i$, $i=1,\ldots,p$.  Now consider a rotation by $V^T$ in \eqref{eq:nest_ode}, i.e., multiply both sides of \eqref{eq:nest_ode} on the left by $V^T$.  Then, defining $\alpha(t) = V^T \beta(t)$ and $\alpha_0 = V^T \beta_0$, we obtain the decoupled system
\begin{equation}
  \ddot \alpha(t) + \frac{3}{t}\dot\alpha(t) + S\alpha(t) = n^{-1/2} (S^{1/2})^T U^T y, \label{eq:nest_ode_decoupled}
\end{equation}
where $\alpha(0) = \dot \alpha(0) = 0$.

When $s_i = 0$, $i=1,\ldots,p$, \eqref{eq:nest_ode_decoupled} becomes
\[
  \ddot \alpha_i(t) + \frac{3}{t} \dot \alpha_i(t) = 0,
\]
and it can be checked that $a_i(t) = 0$ is the unique solution.  On the other hand, when $s_i > 0$, some tedious calculations (cf.~\citet{SuBoCa14}) show that
\[
  \alpha_i(t) = \frac{v_i}{s_i} \left(1 - 2 \frac{J_1(t \sqrt{s_i})}{t \sqrt{s_i}} \right)
\]
uniquely solves \eqref{eq:nest_ode_decoupled}, where $v_i \in \R^p$ denotes a column of $V$ (eigenvector of $\hat \Sigma$), and $J_1$ denotes the Bessel function of the first kind of order one.

Therefore, putting together the pieces, we may express a solution to \eqref{eq:nest_ode} relatively compactly as
\begin{equation}
  \hat \beta^\Nest(t) = (X^T X)^+ (I - V g^\Nest(S,t) V^T) X^T y, \label{eq:nest_ode_soln}
\end{equation}
for any $t \geq 0$.  Here, $A^+$ denotes the Moore-Penrose pseudo-inverse of $A$, and importantly $\smash{g^\Nest : \R_+^{p \times p} \times \R_+ \to \R_+^{p \times p}}$ is a shrinkage map, i.e.,
\begin{equation} \label{eq:nest_ode_shrinkage}
  g^\Nest_{ij}(S,t) =
  \begin{cases}
    2 \frac{J_1(t \sqrt{s_i})}{t \sqrt{s_i}}, & i = j \\
    0, & i \neq j.
  \end{cases}
\end{equation}

The following result leverages the special structure in \eqref{eq:nest_ode_soln}, \eqref{eq:nest_ode_shrinkage} to characterize $\Risk(\hat \beta^\Nest(t); \beta_0)$.

\begin{lemma}[Exact risk of accelerated gradient flow] \label{lem:nest_ode_soln_risk}
  Assume the data model \eqref{eq:data_model}.  Fix $\beta_0 \in \R^p$.  Then, for $t \geq 0$, the risk of accelerated gradient flow \eqref{eq:nest_ode_soln} is
  \begin{equation} \label{eq:nest_ode_risk}
    \begin{aligned}
      & \Risk(\hat \beta^\Nest(t); \beta_0) \\
      & \quad = \sum_{i=1}^p \Bigg( \underbrace{4 (v_i^T \beta_0)^2 \frac{(J_1(t \sqrt{s_i}))^2}{t^2 s_i}}_{\textnormal{Bias}} + \underbrace{\frac{\sigma^2}{n} \frac{\left(1 - 2 \frac{J_1(t \sqrt{s_i})}{t \sqrt{s_i}} \right)^2}{s_i}}_{\textnormal{Variance}} \Bigg).
    \end{aligned}
  \end{equation}
Here and below, we take by convention $(J_1(t \sqrt{x}))^2/(t^2 x) = 1/4$ and $(1 - 2 J_1(t \sqrt{x})/(t \sqrt{x}))^2/x = 0$, when $x = 0$.
\end{lemma}

\begin{proof}
We prove a slightly more general result.  Fix $t \geq 0$.  Consider an estimator $\hat \beta(t) \in \R^p$ given by
\[
  \hat \beta(t) = (X^T X)^+ (I - V g(S,t) V^T) X^T y,
\]
where $g(S,t) \in \R^{p \times p}_+$ is a diagonal shrinkage map.
Now recall the usual bias-decomposition, i.e.,
\[
  \Risk(\hat \beta(t); \beta_0) = \Bias^2(\hat \beta(t); \beta_0) + \Var(\hat \beta(t)),
\]
where we write
\begin{align*}
  \Bias^2(\hat \beta(t); \beta_0) & = \| \E[\hat \beta(t)] - \beta_0 \|_2^2, \\
  \Var(\hat \beta(t)) & = \tr \big( \Cov(\hat \beta(t)) \big).
\end{align*}
For the bias, a few elementary calculations show:
\begin{align*}
  & \Bias^2(\hat \beta(t); \beta_0) \\
  & \quad = \| \E [ (X^T X)^+ (I - V g(S,t) V^T) X^T y ] - \beta_0 \|_2^2 \\
  & \quad = \sum_{i=1}^p (v_i^T \beta_0)^2 g_{ii}^2.
\end{align*}
Similarly, for the variance, we have that $\Var(\hat \beta(t)) = \sum_{i=1}^p \sigma^2 (1-g_{ii})^2 / (n s_i)$.  Substituting $\hat \beta^\Nest(t)$ and $g^\Nest$ in for $\hat \beta(t)$ and $g$ gives the result.
\end{proof}

\subsection{Risk Inflation Under Oracle Tuning}
Now recall the exact risk of the ridge estimator \eqref{eq:ridge} (cf.~\citet{HsuKaZh12b,DobribanWa18,AliKoTi19}), i.e., for any $\lambda \geq 0$, we have
\begin{equation} \label{eq:ridge_risk}
  \begin{aligned}
    & \Risk(\hat \beta^\ridge(\lambda); \beta_0) \\
    & \quad = \sum_{i=1}^p \Bigg( \underbrace{(v_i^T \beta_0)^2 \frac{\lambda^2}{(s_i + \lambda)^2}}_{\textnormal{Bias}} + \underbrace{\frac{\sigma^2}{n} \frac{s_i}{(s_i + \lambda)^2}}_{\textnormal{Variance}} \Bigg).
  \end{aligned}
\end{equation}
Prior work (Theorem 3 in \citet{AliKoTi19}) has shown that the optimal ridge risk using the oracle tuning parameter $\lambda^* = \sigma^2 p / (r^2 n)$ admits a tight coupling with that of early-stopped gradient flow \eqref{eq:gf} for the (unregularized) least squares problem \eqref{eq:ls}, at time $t = \tau / \lambda^*$ with $\tau = 1$.  As it turns out, there is really nothing special about the value $\tau = 1$, which the following result demonstrates.  The result strengthens Theorem 3 from \citet{AliKoTi19}, and is useful in the current paper for giving a tight coupling between accelerated gradient flow and ridge under oracle tuning, which we do next.  We write $\Risk(\hat \beta) = \E \| \hat \beta - \beta_0 \|_2^2$ for the Bayes risk of $\hat \beta \in \R^p$.  

\begin{lemma}[Risk inflation for gradient flow] \label{lem:gf_vs_ridge_under_optimal_tuning}
  Assume the data model \eqref{eq:data_model}.  For any $t \geq 0$, denote the unique solution to \eqref{eq:gf} by
  \begin{equation}
    \hat \beta^\gf(t) = (X^T X)^+ (I - V g^\gf(S,t) V^T) X^T y, \label{eq:gf_soln}
  \end{equation}
  with the diagonal shrinkage map
  \begin{equation} \label{eq:gf_shrinkage}
    \begin{aligned}
      g_{ij}^\gf(S,t) =
      \begin{cases}
        \exp(-t s_i), & i=j \\
        0, & i \neq j.
      \end{cases}
    \end{aligned}
  \end{equation}
  Then, it holds for the gradient flow estimator \eqref{eq:gf_soln} that
  \[
    1 \leq \frac{\inf_{t \geq 0} \Risk(\hat \beta^\gf(t))}{\inf_{\lambda \geq 0} \Risk(\hat \beta^\ridge(\lambda))} \leq 1.0786.
  \]
\end{lemma}

\begin{proof}
Let $\alpha = r^2 n/(\sigma^2 p)$.  The Bayes risk of gradient flow was given in Lemma 5 of \citet{AliKoTi19}:
\begin{equation} \label{eq:gf_ode_risk}
  \begin{aligned}
    & \textnormal{Risk}(\hbeta^{\textnormal{gf}}(t)) \\
    & = \frac{\sigma^2}{n}\sum_{i=1}^p\left(\alpha\exp(-2ts_i)+\frac{(1-\exp(-ts_i))^2}{s_i}\right),
  \end{aligned}
\end{equation}
for any $t \geq 0$.  Meanwhile, the optimal Bayes risk for ridge is well-known, i.e.,
\begin{equation}
  \inf_{\lambda \geq 0} \Risk(\hat \beta^\ridge(\lambda)) = \frac{\sigma^2}{n} \sum_{i=1}^p \frac{\alpha }{\alpha s_i + 1}, \label{eq:ridge_opt_risk}
\end{equation}
where $\lambda^* = \sigma^2 p / (r^2 n)$ is the optimal regularization strength.  Now the lower bound follows because $\hat \beta^\ridge(\lambda^*)$ is the Bayes estimator under \eqref{eq:data_model}, i.e., \eqref{eq:ridge_opt_risk} must be smaller than \eqref{eq:gf_ode_risk}, for any $t \geq 0$ (see, e.g., Theorem 3 in \citet{AliKoTi19} for a discussion).

As for the upper bound, we compare, for $s \geq 0$,
\[
  \alpha\exp(-2ts)+\frac{(1-\exp(-ts))^2}{s} \quad \textrm{and} \quad \frac{\alpha}{\alpha s + 1}.
\]
Let $\tau \geq 0$ be arbitrary, and plug $t = \tau / \lambda$ into the first term.  Changing variables and dividing the two terms together gives rise to the following min-max problem, which clearly upper bounds the risk ratio:
\begin{equation*}
  \begin{aligned}
    & \min_{\tau\geq0}\max_{x\geq0}\Big\{(1+x)\exp(-2\tau x)\\
    & \hspace{1.25in} +\frac{(1+x)(1-\exp(-\tau x))^2}{x}\Big\}.
  \end{aligned}
\end{equation*}
We can numerically compute the optimal value over a finite set of $\tau$ values, giving the result.
\end{proof}

Assuming optimal tuning, Lemma \ref{lem:gf_vs_ridge_under_optimal_tuning} says that there is really no difference between the (explicitly regularized) ridge and (implicitly regularized) gradient flow estimators \eqref{eq:ridge}, \eqref{eq:gf}, respectively, at least as far as estimation error goes.  Moreover, the lemma strengthens Theorem 3 from \citet{AliKoTi19}, by delivering the sharper constant 1.0786 (cf.~1.2147, from the theorem); by construction, this is the sharpest possible constant available, at least with the strategy used to prove these two results.

Finally, with the above results in hand, we turn to establishing the analog of Lemma \ref{lem:gf_vs_ridge_under_optimal_tuning} for accelerated gradient flow; the following result gives the details.

\begin{theorem}[Risk inflation for accelerated gradient flow] \label{lem:nest_ode_vs_ridge_under_optimal_tuning}
  Assume the data model \eqref{eq:data_model}.  Then, it holds for the accelerated gradient flow estimator \eqref{eq:nest_ode_soln} that
  \[
    1 \leq \frac{\inf_{t \geq 0} \Risk(\hat \beta^\Nest(t))}{\inf_{\lambda \geq 0} \Risk(\hat \beta^\ridge(\lambda))} \leq 1.5991.
  \]
\end{theorem}

\begin{proof}
The strategy is roughly the same as that for the proof of Lemma \ref{lem:gf_vs_ridge_under_optimal_tuning}.  Let $\alpha = r^2 n/(\sigma^2 p)$.  Simply taking expectations in \eqref{eq:nest_ode_risk} with respect to $\beta_0$ shows that the Bayes risk of $\hat \beta^\Nest(t)$ satisfies, for any $t \geq 0$,
\begin{equation*} 
  \begin{aligned}
    & \Risk(\hat \beta^\Nest(t)) \\
    & = \frac{\sigma^2}{n} \sum_{i=1}^p \left( \alpha \frac{(2 J_1(t \sqrt{s_i}))^2}{t^2 s_i} + \frac{\left(1 - 2 \frac{J_1(t \sqrt{s_i})}{t \sqrt{s_i}} \right)^2}{s_i} \right).
  \end{aligned}
\end{equation*}
The lower bound follows just as in the proof of Lemma \ref{lem:gf_vs_ridge_under_optimal_tuning}.  For the upper bound, we now compare
\[
  \alpha \frac{(2 J_1(t \sqrt{x}))^2}{t^2 x} + \frac{\left(1 - 2 \frac{J_1(t \sqrt{x})}{t \sqrt{x}} \right)^2}{x} \quad \textrm{and} \quad \frac{\alpha }{\alpha x + 1},
\]
for $x \geq 0$.  Letting $\tau \geq 0$, plugging $t = \tau \lambda^{-1/2}$ into the first term, and dividing, we get
\begin{equation*}
  \min_{\tau\geq0}\max_{x\geq0}\left\{\frac{4(1+x)J_1^2(\tau\sqrt{x})}{\tau^2x}+\frac{(1+x)\left(1-\frac{2J_1(\tau\sqrt{x})}{\tau\sqrt{x}}\right)^2}{x}\right\}.
\end{equation*}
The result follows by numerical maximization.
\end{proof}

\subsection{Bias, Variance, and Stability}
Though Theorem \ref{lem:nest_ode_vs_ridge_under_optimal_tuning} shows the optimal ridge and accelerated risks are similar, we might hope for a bound holding \textit{uniformly} over $t \geq 0$, assuming a suitable relation between $t$, $\lambda$.  Unfortunately, such a result is not possible in general, as we now show.

It is natural to assume the relation $t = \lambda^{-1/2}$.
Now from looking back at the bias-variance decompositions that were given in \eqref{eq:ridge_risk}, \eqref{eq:nest_ode_risk}, we would like to establish the existence of two universal constants $C_1, C_2 > 0$, such that
\begin{align*}
  \frac{(2 J_1(\sqrt{x}))^2}{x} & \leq C_1 \frac{1}{(1+ x )^2} \\
  \left(1 - \frac{2 J_1(\sqrt{x})}{\sqrt{x}}\right)^2 & \leq C_2 \frac{x^2}{(1 + x)^2},
\end{align*}
where $x = t^2 s$, in order to prove the required bound.  However, as $J_1(x) = O(x^{-1/2})$, we see $4 J_1^2(\sqrt{x})/x = O(x^{-3/2})$, and therefore $O(x^2 \cdot x^{-3/2})$ grows without bound, i.e., finding such a $C_1$ is impossible.

Simple empirical examples demonstrate and illuminate the issue.  In Figure \ref{fig:stability_poly}, we plot the bias, variance, and risk of accelerated gradient flow \eqref{eq:nest_ode_risk}, heavy ball flow \eqref{eq:hb_ode_risk}, standard gradient flow (see, e.g., Lemma 5 in \citet{AliKoTi19}), and ridge regression \eqref{eq:ridge_risk}, for data arising from the response model \eqref{eq:data_model} with $n = 500$, $p = 100$, and sample covariance matrices $\hat \Sigma$ having singular values satisfying $s_i = C/i^\nu$, $i=1,\ldots,p$, with $C = 1$ and $\nu \in \{0.1, 0.5, 1, 2\}$.  We generated (dense) underlying coefficients $\beta_0$, such that the signal-to-noise ratio $\| \beta_0 \|_2^2/\sigma^2 = 1$.  Below, we make a few observations.
\begin{itemize}
  \item When $\nu \in \{0.1, 0.5\}$ is small, i.e., there are many sizable eigenvalues, then even with small $t$ (as in the first row of the figure), we see that the ridge bias drops faster than the accelerated bias, reflecting the situation described earlier.

  \item On the other hand, when $t$ is large, even though we may have $\nu \in \{1, 2\}$, i.e., many small eigenvalues (as in the last row of the figure), we see the roles of the two biases reverse, so that now the accelerated bias drops faster.
  
  This second situation, i.e., when the design matrix is ill-conditioned, is interesting to inspect further.  First of all, a simple Taylor expansion of the bias terms in \eqref{eq:nest_ode_risk}, \eqref{eq:gf_ode_risk} confirms the situation: when $s \approx 0$ and $t$ is large enough, then we have $2 J_1(t \sqrt{s})/(t \sqrt{s}) \ll \exp(-t s)$.  But a quickly diminishing bias implies a quickly growing variance (e.g., the last two rows), revealing a kind of instability due to acceleration that is more subtle than has been described previously \citep{ChenJiYu18,PaglianaRo19}.
  
  \item We comment on one obvious aspect of the plots.  The fact that accelerated methods are not descent methods implies their bias and variance are not monotone, and therefore that the accelerated risk can (and does) oscillate.  These oscillations are especially pronounced in the ill-conditioned regime (the last two rows), indicating that early stopping should be used carefully.

  How do we reconcile this last observation with the good performance of acceleration seen in practice (e.g., \citet{WilsonRoStSrRe17}), when data is high-dimensional?  Well, in these settings, an interpolating solution is generally sought \citep{HastieMoRoTi19}, circumventing some of the issues described above.  But when early stopping is employed with acceleration in practice, the stopping point is usually where the hold-out set error curve attains its minimum, i.e., roughly what we expect from the oracle tuning parameter (recall Theorem \ref{lem:nest_ode_vs_ridge_under_optimal_tuning}).
\end{itemize}

As a whole, the story here appears to be more complex and subtle than with standard gradient descent.

\begin{figure}[h!]
  \centering{
    \includegraphics[scale=0.4]{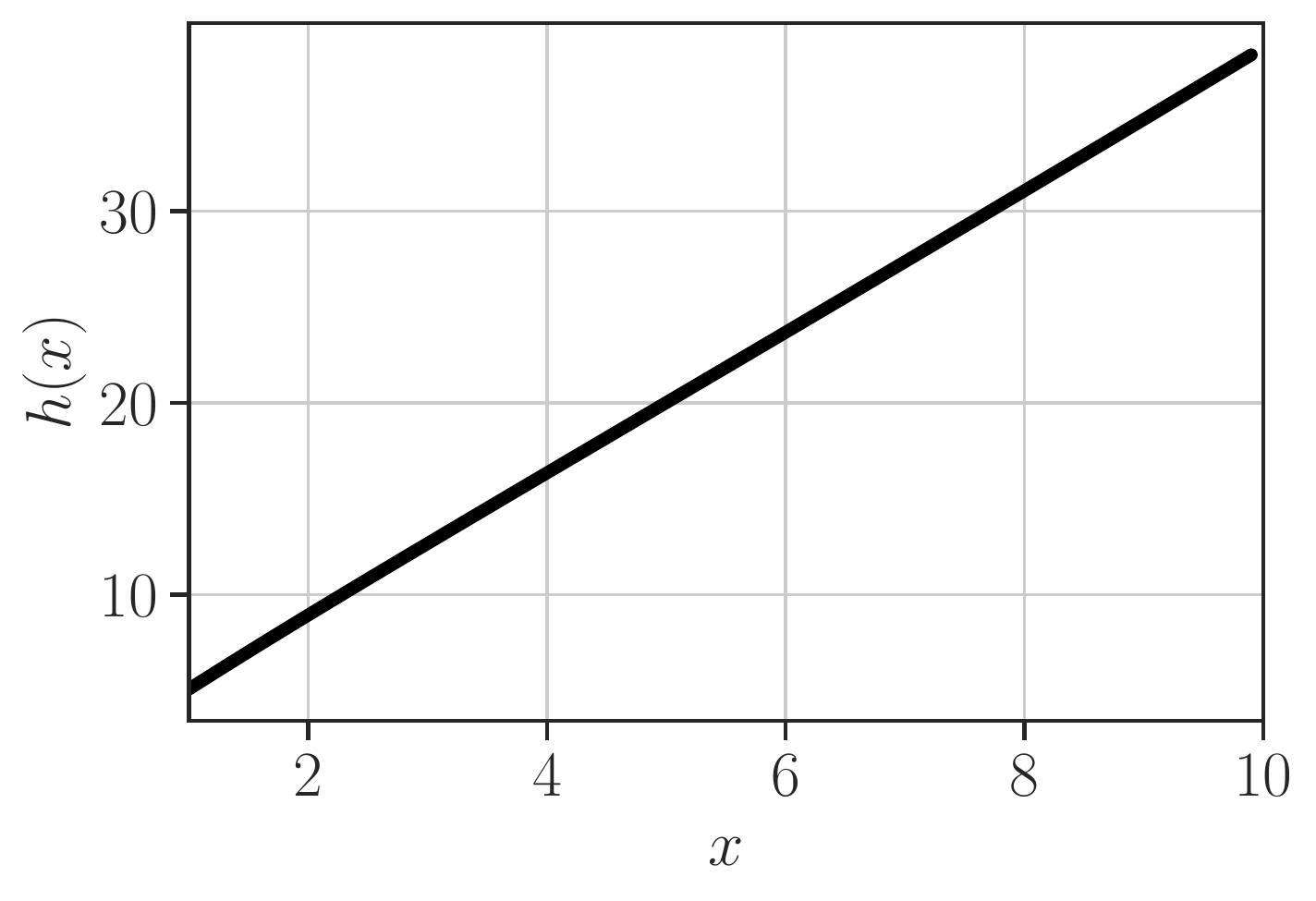}
  }
  \vspace{-.1in}
  \caption{\textit{The function $h(x)$ from Theorem \ref{lem:hb_vs_ridge_under_optimal_tuning}.}}
  \label{fig:h}
\end{figure}


\begin{figure*}[h!]
\vspace{.3in}
\centering{
  \includegraphics[scale=0.23]{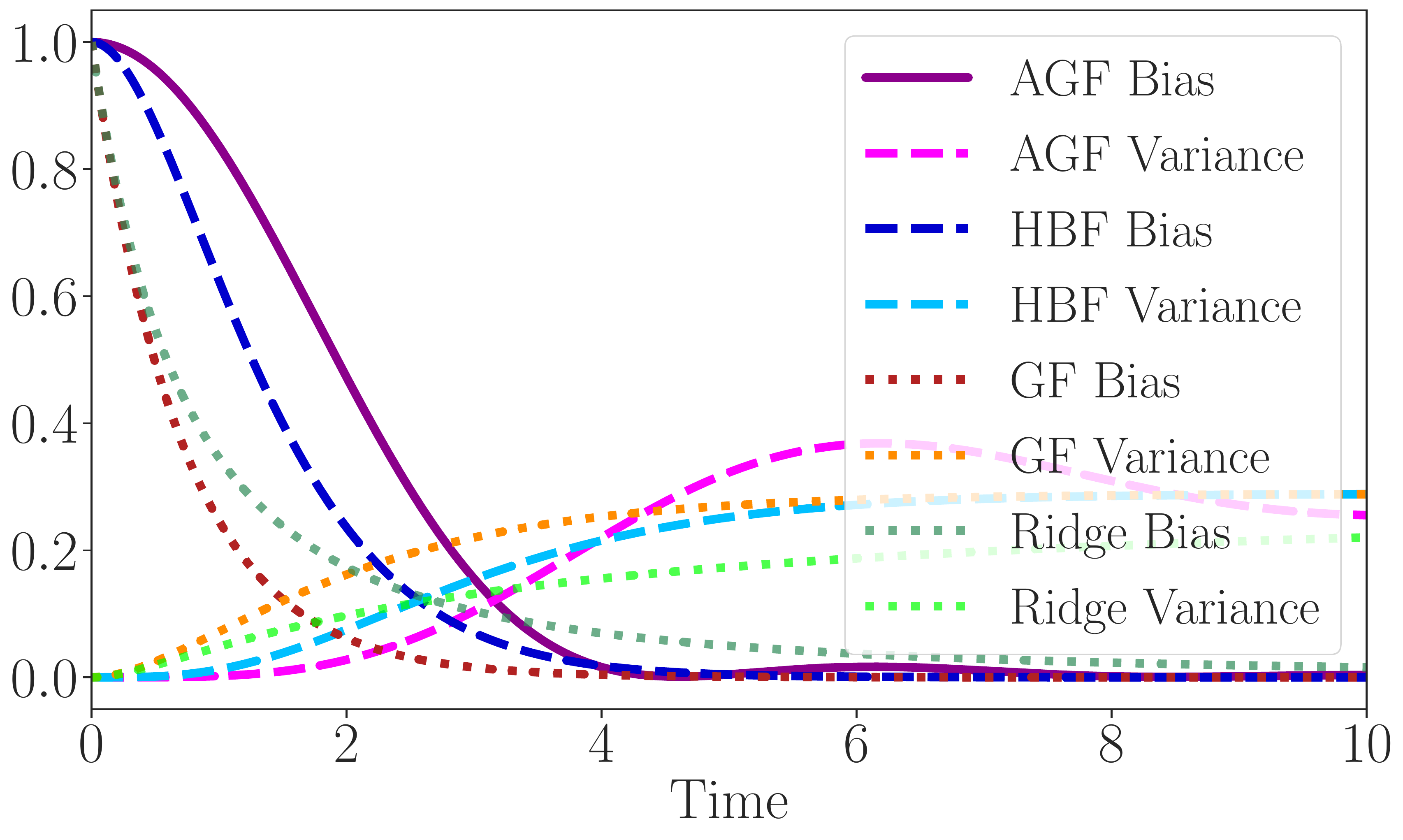} \hfill
  \includegraphics[scale=0.23]{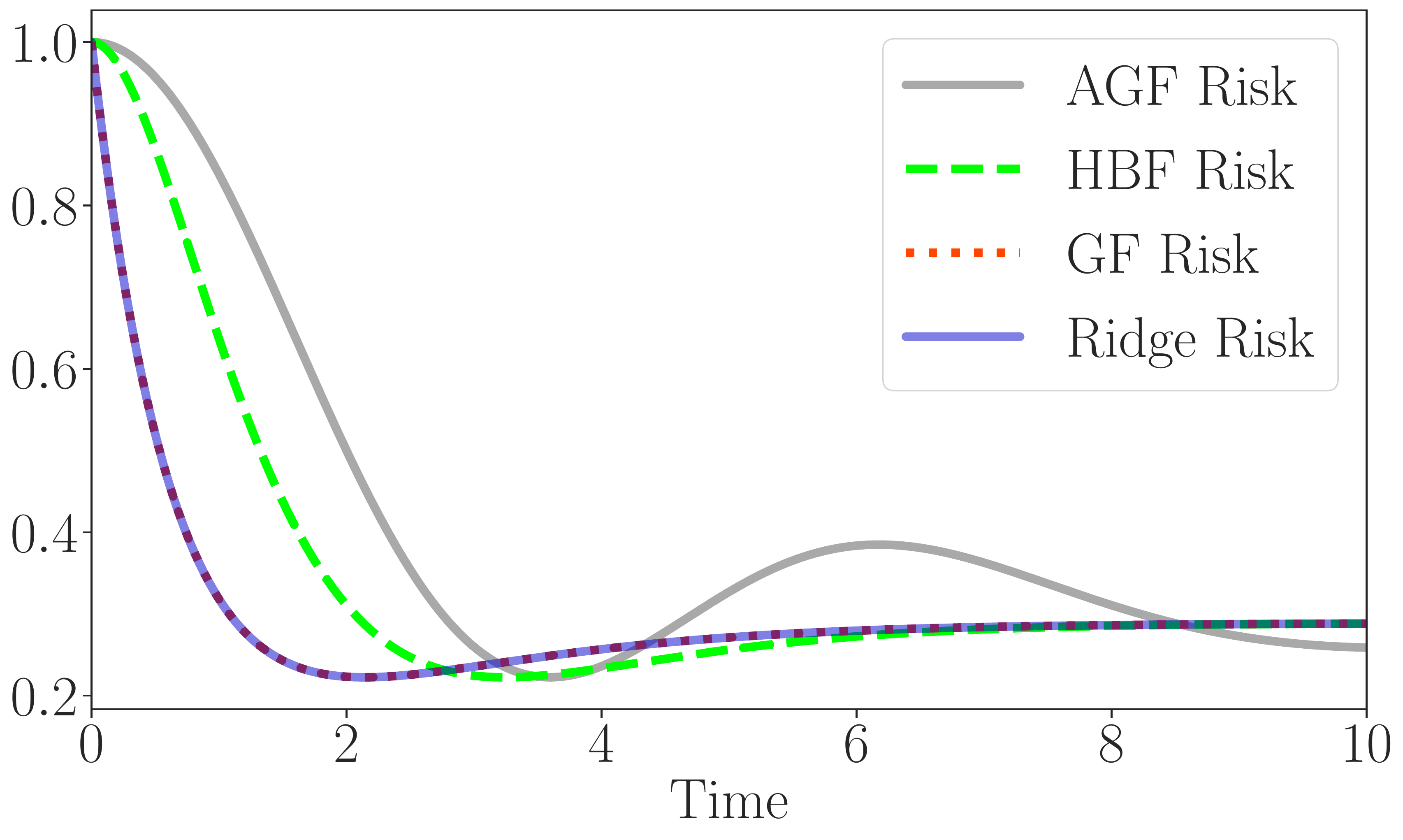} \\
  \includegraphics[scale=0.23]{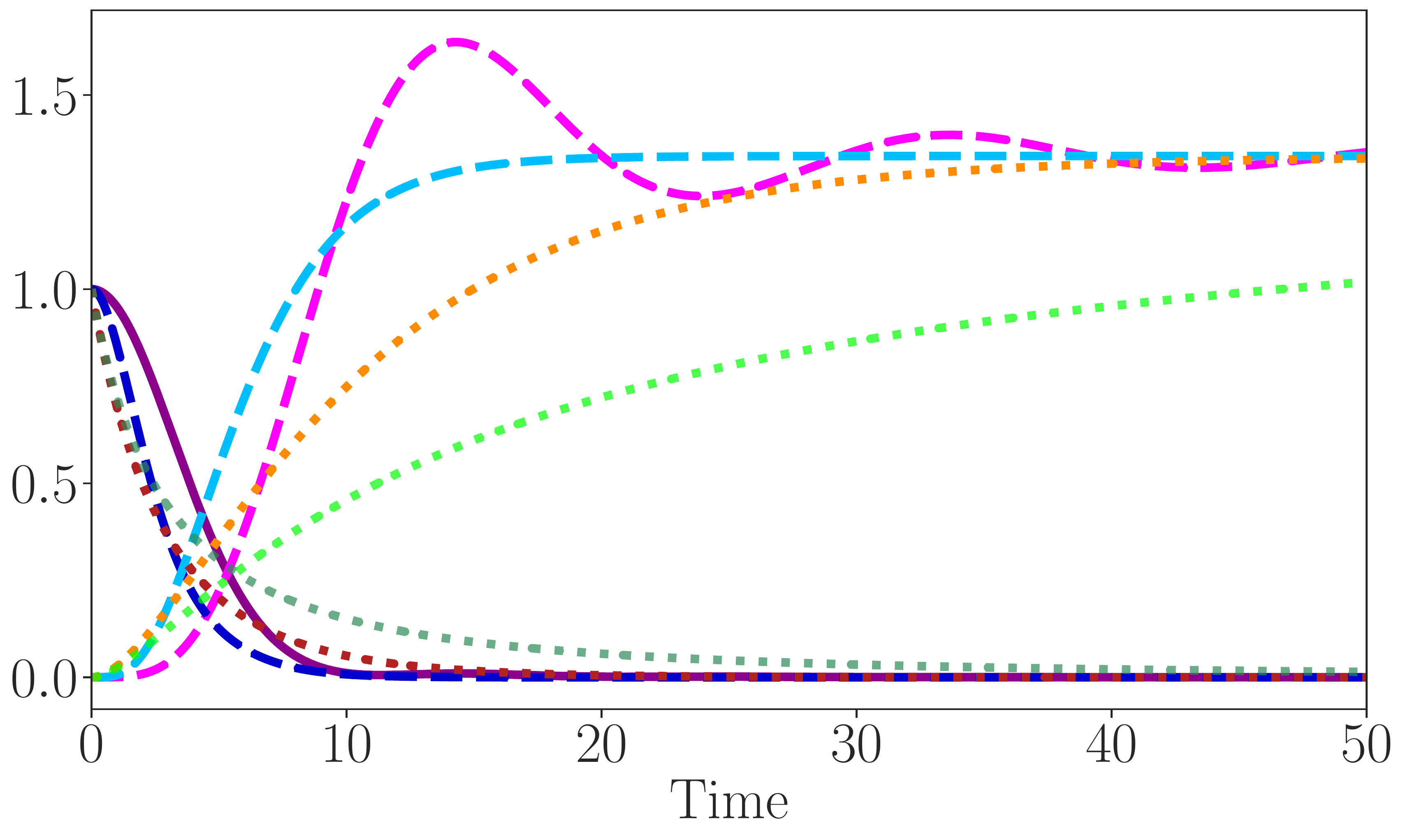} \hfill
  \includegraphics[scale=0.23]{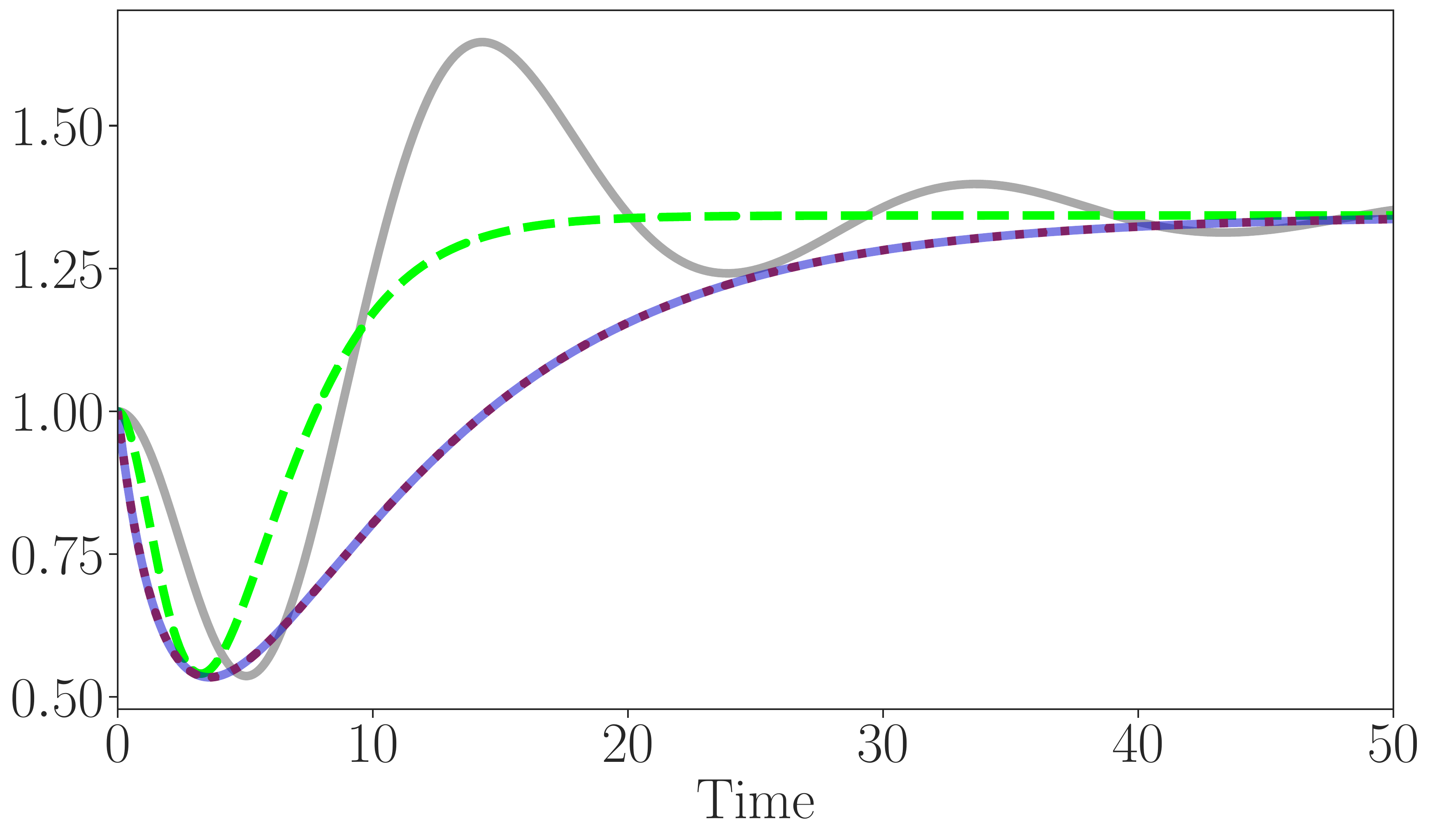} \\
  \includegraphics[scale=0.23]{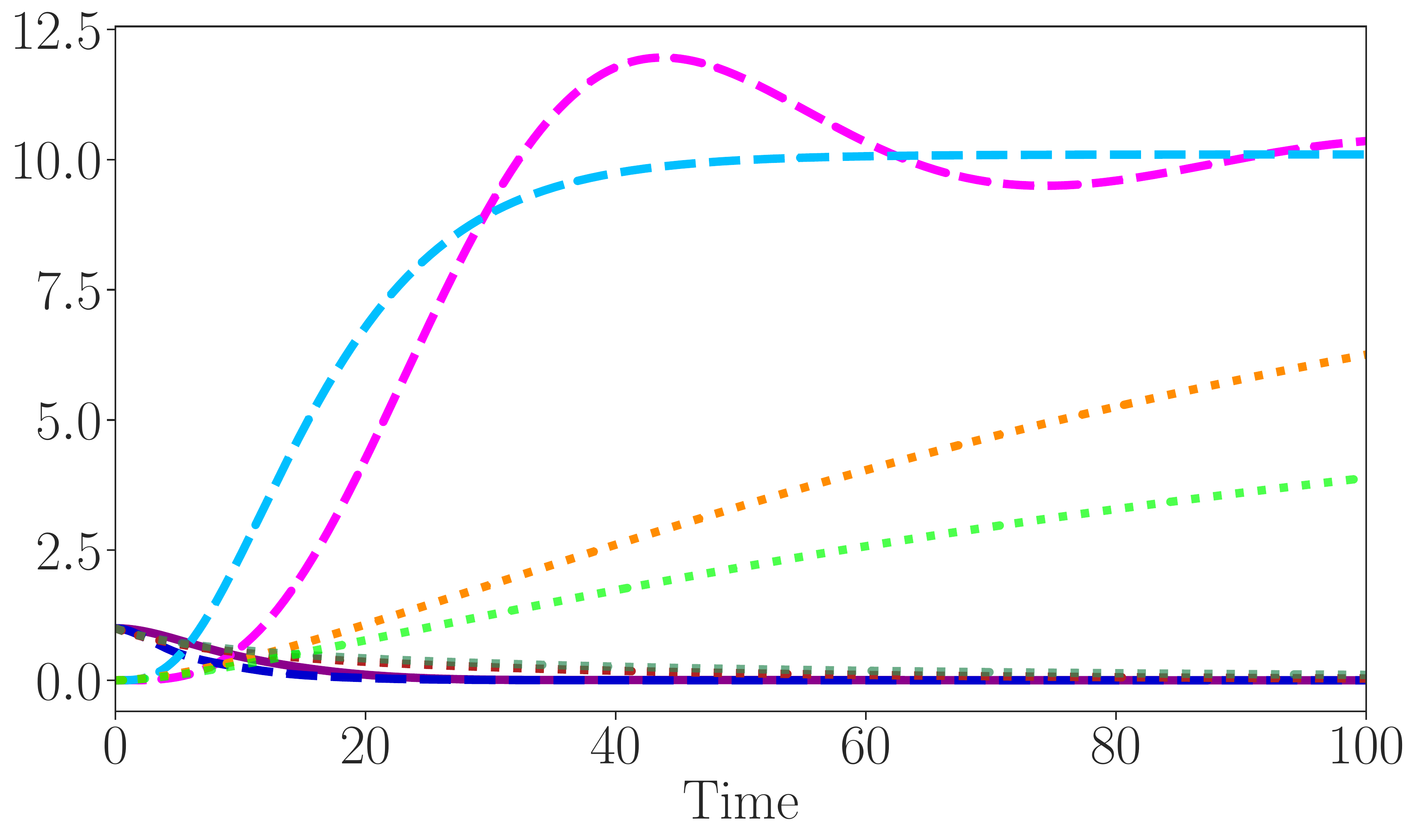} \hfill
  \includegraphics[scale=0.23]{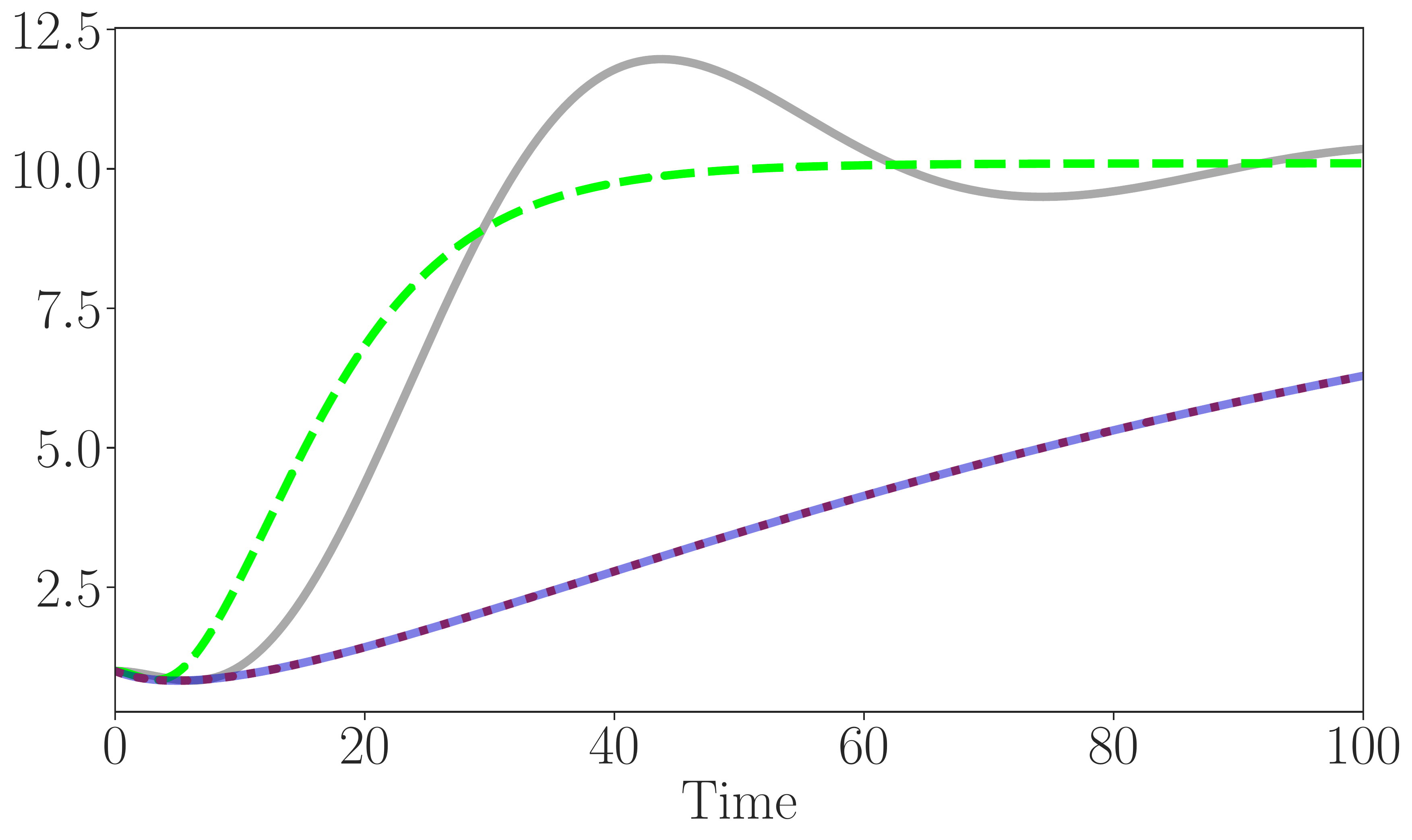} \\
  \includegraphics[scale=0.23]{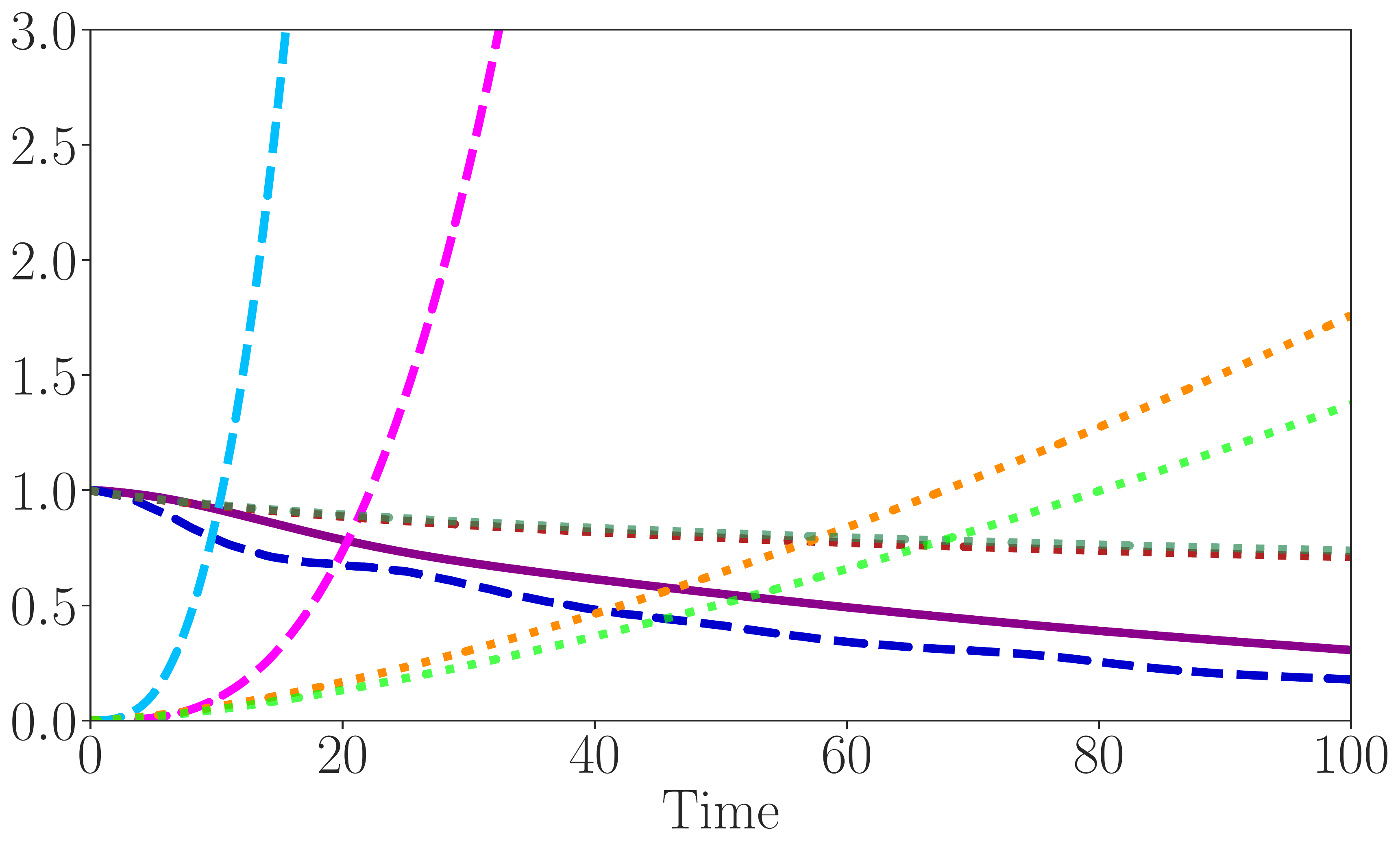} \hfill
  \includegraphics[scale=0.23]{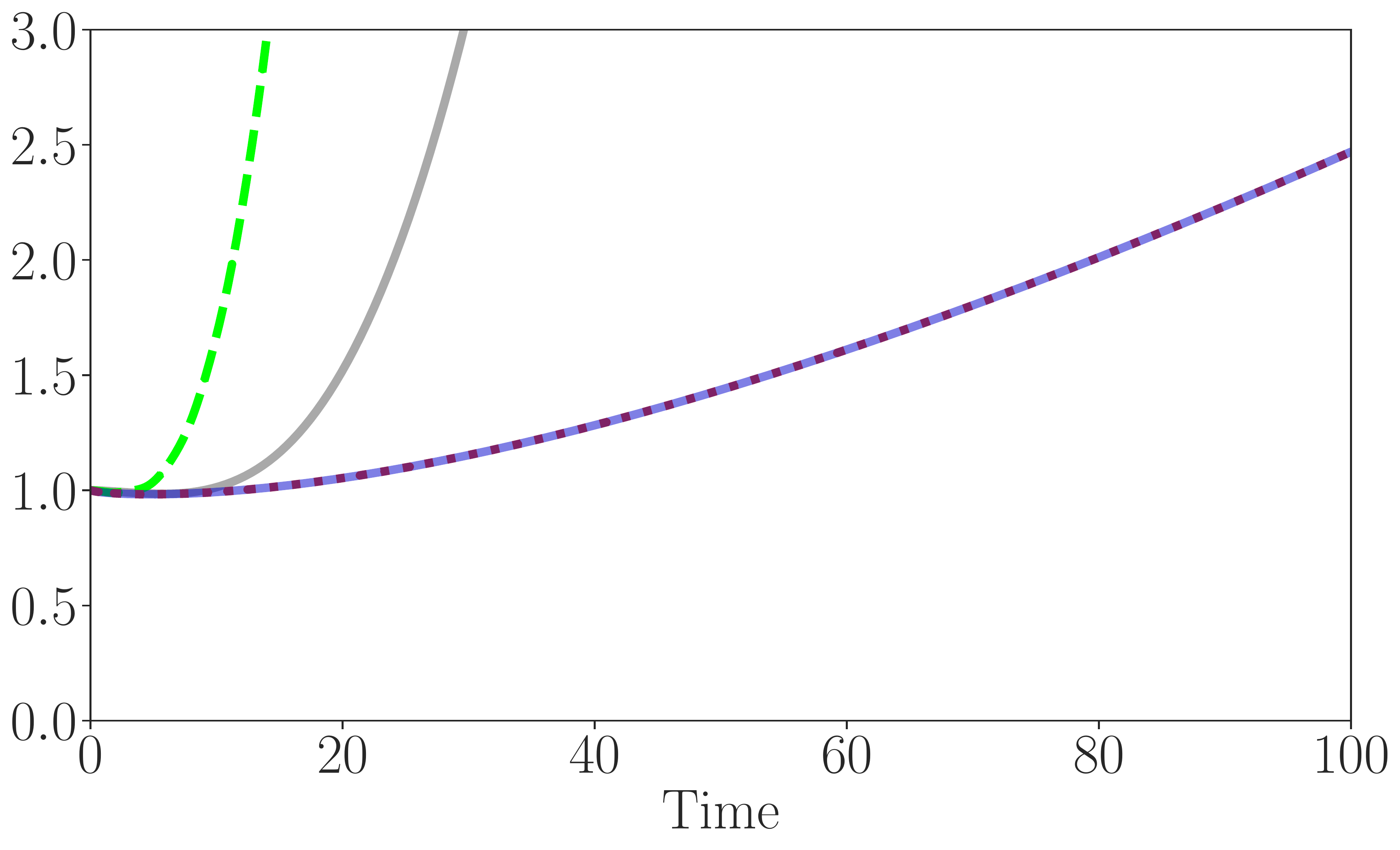}
}
\vspace{.3in}
\caption{\textit{Bias, variance (left column), and risk (right column) for accelerated gradient flow \eqref{eq:nest_ode_risk}, heavy ball flow \eqref{eq:hb_ode_risk}, standard gradient flow (see, e.g., Lemma 5 in \citet{AliKoTi19}), and ridge regression \eqref{eq:ridge_risk}.  Each row corresponds to a different sample covariance matrix with its eigenvalues set according to $s_i = C/i^\nu$, for $i=1,\ldots,p$, with $C = 1$ and $\nu \in \{0.1, 0.5, 1, 2\}$, from top to bottom, respectively.}}
\label{fig:stability_poly}
\end{figure*}

\subsection{Parameter Error Bound}
Although we cannot give a pointwise risk coupling between accelerated gradient flow and ridge (as was just discussed), we can still give a \textit{relative} bound on the error $\E \| \hat \beta^\Nest(t) - \hat \beta^\ridge(\lambda) \|_2^2$, assuming the natural correspondence $t = \lambda^{-1/2}$.  Our final result for this section (below) delivers such a bound.

\begin{theorem}[Parameter error bound] \label{thm:nest_ode_param_error}
  Assume the data model \eqref{eq:data_model}.  For any $t \geq 0$, it holds that
  \[
    \E \| \hat \beta^\Nest(t) - \hat \beta^\ridge(1/t^2) \|_2^2 \leq 0.7657 \cdot \E \| \hat \beta^\ridge(1/t^2) \|_2^2.
  \]
\end{theorem}

The proof of the theorem is in the appendix.  The bound tells us the parameter error is well-controlled, as the right-hand side is small when $t$ is small and tends quickly to a constant value as $t$ grows.

\section{HEAVY BALL METHOD}
\label{sec:hb}

Now we turn to analyzing the heavy ball method.  In what follows, we assume that $\mu > 0$.

In analogy to both \eqref{eq:nest_ode_soln} and Lemma \ref{lem:nest_ode_soln_risk}, our first result gives a closed-form solution to the heavy ball differential equation, from \eqref{eq:hb_ode}, along with a precise characterization of its estimation risk.  The proof of the result is straightforward and in the appendix.

\begin{lemma}[Solution and risk of heavy ball flow] \label{lem:hb_ode_soln_risk}
  Assume the data model \eqref{eq:data_model}.  Then, for any $t \geq 0$:
  \begin{itemize}
    \item the heavy ball flow estimator, defined as
    \begin{equation}
      \hat \beta^\hb(t) = (X^T X)^+ (I - V g^\hb(S,t)) V^T) X^T y, \label{eq:hb_ode_soln}
    \end{equation}
    where the shrinkage map has entries
    \begin{equation} \label{eq:hb_ode_shrinkage}
      \begin{aligned}
      & g^\hb_{ij}(S,t) = \exp(-\mu^{1/2} t) \\
      & \times \Bigg( \cos(t(s_i - \mu)^{1/2}) + \frac{\mu^{1/2} \sin(t(s_i - \mu)^{1/2})}{(s_i-\mu)^{1/2}} \Bigg),
      \end{aligned}
    \end{equation}
    for $i = j$, and zeros everywhere else, solves \eqref{eq:hb_ode};
    
    \item the risk of the heavy ball flow estimator \eqref{eq:hb_ode_soln} is
    \begin{equation} \label{eq:hb_ode_risk}
      \begin{aligned}
        & \Risk(\hat \beta^\hb(t); \beta_0) = \\
        & \sum_{i=1}^p \Bigg( \underbrace{(v_i^T \beta_0)^2 (1 - g^\hb_{ii}(S,t))^2}_{\textnormal{Bias}} + \underbrace{\frac{\sigma^2}{n} \frac{(g^\hb_{ii}(S,t))^2}{s_i}}_{\textnormal{Variance}} \Bigg).
      \end{aligned}
    \end{equation}
  \end{itemize}
\end{lemma}

The next result bounds the risk inflation of the heavy ball flow estimator \eqref{eq:hb_ode_soln} over ridge \eqref{eq:ridge}, albeit by a worse absolute constant than in Theorem \ref{lem:nest_ode_vs_ridge_under_optimal_tuning}.  The shrinkage map \eqref{eq:hb_ode_shrinkage} is considerably more complex than, say, \eqref{eq:nest_ode_shrinkage}, which makes the analysis difficult.  We suspect the constant is suboptimal and due to our proof technique, as numerical examples (see, e.g., Figure \ref{fig:stability_poly}) indicate \eqref{eq:hb_ode_soln}, \eqref{eq:nest_ode_soln} behave similarly.

\begin{lemma}[Risk inflation for heavy ball flow] \label{lem:hb_vs_ridge_under_optimal_tuning}
  Assume the data model \eqref{eq:data_model}.  Let $\kappa = L/\mu$, and define
  \begin{equation}
    \begin{aligned}
      & h(x) = 8 x^{2/3} + \left( 1 + \left( \frac{x^{1/3} + \sqrt{5 x^{2/3} - 4}}{2}\right)^2 \right) \times \\
      & \left( \frac{\sqrt{5 x^{2/3} - 4}}{2 x^{1/3}} + \frac{3}{2} \right)^2 \exp\left( - \frac{x^{1/3} + \sqrt{5 x^{2/3} - 4}}{x^{1/3}} \right).
    \end{aligned}
  \end{equation}
  Then, it holds for heavy ball flow \eqref{eq:hb_ode_soln} that
  \[
    1 \leq \frac{\inf_{t \geq 0} \Risk(\hat \beta^\hb(t))}{\inf_{\lambda \geq 0} \Risk(\hat \beta^\ridge(\lambda))} \leq h(\kappa).
  \]
\end{lemma}

The proof is lengthy; see the appendix.  To get a sense of the result, we plot $h(x)$ in Figure \ref{fig:h}.

Our last result delivers a parameter error bound, similar to Theorem \ref{thm:nest_ode_param_error}; its proof is in the appendix.

\begin{theorem}[Parameter error bound for heavy ball] \label{thm:hb_ode_param_error}
  Assume the data model \eqref{eq:data_model}.  For $t \geq 0$, it holds that
  \[
    \E \| \hat \beta^\hb(t) - \hat \beta^\ridge(1/t^2) \|_2^2 \leq 25 \cdot \E \| \hat \beta^\ridge(1/t^2) \|_2^2.
  \]
\end{theorem}

\section{ADDITIONAL EXAMPLES}
\label{sec:exps}
Due to space limitations, we present additional numerical examples with Gaussian, Student-t, and orthogonal data matrices, in the appendix; the takeaway message is similar to that for Figure \ref{fig:stability_poly}.

\section{CONCLUSION}
\label{sec:disc}

We studied the statistical properties of the iterates generated by Nesterov's accelerated gradient method and Polyak's heavy ball method, for least squares regression, making several connections to explicit penalization.  There are a number of directions that would be interesting to pursue, e.g., extending the analysis to cover the high-resolution differential equations of \citet{ShiDuJoSu21}, as well as studying adaptive gradient methods and conjugate gradients through the same lens.

\clearpage

\onecolumn
\section{APPENDIX}

\subsection{Proof of Theorem \ref{thm:nest_ode_param_error}}
Fix $\varepsilon$.  Using the correspondence $t = \lambda^{-1/2}$ and making a few simplifications gives
\begin{equation*}
  \| \hat \beta^\Nest(t) - \hat \beta^\ridge(1/t^2) \|_2^2 = \Big\| n^{-1/2} tV \Big( (t^2S)^+ (I - g^\Nest(S,t)) - ( t^2S + I )^{-1} \Big) ((t^2S)^{1/2})^T U^T y \Big\|_2^2. 
\end{equation*}

So if we could find a finite numerical constant $C > 0$ satisfying
\begin{equation*}
  \big( (t^2S)^+ (I - g^\Nest(S,t)) - ( t^2S + I )^{-1} \big)^2 \preceq \big( C\cdot ( t^2S + I )^{-1} \big)^2,
\end{equation*}
or equivalently
\begin{equation*}
  \big( (t^2S)^+ (I - g^\Nest(S,t)) ( t^2S + I )^{-1} - I \big)^2 \preceq C^2 \cdot I,
\end{equation*}
then we would be done.

Noting that each of the matrices in the last display is diagonal, we define
\begin{equation*}
  f(x)=\left(1-2\frac{J_1(x)}{x}\right)\frac{x^2+1}{x^2}.
\end{equation*}
Numerically maximizing $(f(x)-1)^2$ over $x \geq 0$ gives $(f(x)-1)^2\leq49/64$, which serves as such a constant.  This shows the result.  \qed

\subsection{Proof of Lemma \ref{lem:hb_ode_soln_risk}}
We follow a strategy similar to the one leading up to \eqref{eq:nest_ode_decoupled} for Nesterov's method, i.e., multiplying \eqref{eq:hb_ode} on the left by $V^T$ gives rise to the decoupled system
\begin{equation}
  \ddot \alpha(t) + 2 \mu^{1/2} \dot \alpha(t) - S \alpha_0 - n^{-1/2} S^{1/2} U^T \varepsilon + S \alpha(t) = 0, \label{eq:hb_ode_decoupled}
\end{equation}
with $\alpha(0) = \dot \alpha(0) = 0$.  After some tedious (but standard) calculations, we obtain that a solution to \eqref{eq:hb_ode_decoupled} has the form
\begin{equation}
  \alpha_i(t)=\frac{v_i}{s_i} - \frac{\exp(-t \mu^{1/2}) \cos( t (s_i - \mu)^{1/2}) v_i}{s_i} - \frac{\exp(-t \mu^{1/2}) \sin( t (s_i - \mu)^{1/2}) v_i \mu^{1/2}}{s_i (s_i - \mu)^{1/2}}, \quad i=1,\ldots,p.
\end{equation}
It follows that we can write a solution to \eqref{eq:hb_ode} as
\[
  \hat \beta^\hb(t) = (X^T X)^+ (I - V g^\hb(S,t)) V^T) X^T y,
\]
with the diagonal shrinkage map
\[
  g^\hb_{ij}(S,t) = 
  \begin{cases}
    \exp(-\mu^{1/2} t) \cdot \Bigg( \cos(t(s_i - \mu)^{1/2}) + \frac{\mu^{1/2} \sin(t(s_i - \mu)^{1/2})}{(s_i-\mu)^{1/2}} \Bigg), & i = j \\
    0, & i \neq j,
  \end{cases}
\]
just as claimed.  Invoking the risk expansion given in the proof of Lemma \ref{lem:nest_ode_soln_risk} immediately proves the second part of the result.  \qed

\subsection{Proof of Lemma \ref{lem:hb_vs_ridge_under_optimal_tuning}}
Looking back at the proofs of Lemmas \ref{lem:nest_ode_soln_risk} and \ref{lem:gf_vs_ridge_under_optimal_tuning} in the main paper, we can see that proving the required result boils down to establishing the inequality
\begin{equation}
  \alpha g^\hb(s,t)^2 + \frac{(1 - g^\hb(s,t))^2}{s} \leq C \cdot \frac{\alpha}{\alpha s + 1}, \label{eq:predict:bound}
\end{equation}
for fixed $s,t > 0$ and some finite numerical constant $C > 0$.

Let $\tau > 0$ be a constant to be specified later.  Consider the change of variables $t = \tau \alpha^{1/2}$, $x = (\alpha s)^{1/2}$ $a = (\alpha \mu)^{1/2}$, and $b = (\alpha (s - \mu))^{1/2}$.  Plugging these and the definition of the heavy ball shrinkage map \eqref{eq:hb_ode_shrinkage} into \eqref{eq:predict:bound}, and rearranging, we see that showing \eqref{eq:predict:bound} is equivalent to showing
\begin{equation}
  (1+x^2)\left(e^{-\tau a}\left(\cos \tau b+\frac{a}{b}\sin \tau b\right)\right)^2+\frac{1+x^2}{x^2}\left(1-e^{-\tau a}\left(\cos \tau b+\frac{a}{b}\sin \tau b\right)\right)^2\leq C. \label{eq:heavy:predict}
\end{equation}
We proceed by bounding each of the two terms (i.e., the bias and the variance) on the left-hand side separately.

\paragraph{Bounding the bias.}  The helper Lemma \ref{lem:hb_vs_ridge_under_optimal_tuning_helper_bias} appearing below gives us a bound for the first (bias) term in \eqref{eq:heavy:predict}, i.e., we know that
\begin{equation*}
  \left(e^{-\tau a}\left(\cos \tau b+\frac{a}{b}\sin \tau b\right)\right)^2\leq (\tau x/\sqrt{\kappa}+1)^2e^{-2\tau x/\sqrt{\kappa}}.
\end{equation*}

To further bound the right-hand side, consider the function
\begin{equation*}
  \tilde h(x)=(1+x^2)(\tau x/\sqrt{\kappa}+1)^2e^{-2\tau x/\sqrt{\kappa}}, \quad x \geq 0,
\end{equation*}
Differentiating $\tilde h$ with respect to $x$ gives
\begin{equation*}
  \frac{\partial \tilde h(x)}{\partial x}=-\frac{2x}{z^3}\left(x+z\right)\left(x-\frac{z-\sqrt{5z^2-4}}{2}\right)\left(x-\frac{z+\sqrt{5z^2-4}}{2}\right)e^{-2x/z},
\end{equation*}
where we defined $z = \kappa^{1/2} / \tau$.  Some inspection reveals that there are three cases to consider.
\begin{itemize}
  \item $0<z\leq\frac{2}{\sqrt{5}}$: in this case, it follows that $\tilde h(x)$ attains its maximum over $x \geq 0$ at 0, with $\tilde h(0)=1$.
  \item $\frac{2}{\sqrt{5}}<z<1$: in this case, it follows that both 0 and $x^*=(z+\sqrt{5z^2-4})/2$ are maxima, and it can be checked numerically that $\tilde h(0)>\tilde h(x^*)$ for $z<0.907$, but $\tilde h(0)\leq \tilde h(x^*)$ for $z\geq 0.907$.
  \item $z\geq1$: in this case, it follows that $\tilde h(x)$ again attains its maximum at $x^*$.
\end{itemize}
To summarize: when $z<0.907$, we have that $\tilde h(x)\leq \tilde h(0)=1$ for $x \geq 0$, but when $z\geq0.907$, we have the bound $\tilde h(x) \leq \tilde h(x^*)$.  As $\kappa \geq 1 > (4/5)^{3/2}$, we may conclude that
\begin{equation}
  \tilde h(x) \leq \tilde h(x^*) = h(x^*) - 8 (x^*)^{2/3}, \quad x \geq 0, \label{eq:hb_vs_ridge_under_optimal_tuning_bias_bound}
\end{equation}
which yields a bound on the bias term in \eqref{eq:heavy:predict}.

\paragraph{Bounding the variance.}  As for the second (variance) term in \eqref{eq:heavy:predict}, the helper Lemma \ref{lem:hb_vs_ridge_under_optimal_tuning_helper_var} appearing below tells us that there are two cases to consider:
\begin{itemize}
  \item for $x \leq 1$, it follows that
  \begin{align}
    \frac{1+x^2}{x^2}\left(1-e^{-\tau a}\left(\cos \tau b+\frac{a}{b}\sin \tau b\right)\right)^2 & \leq \frac{1+x^2}{x^2}\cdot4(\tau x)^4 \notag \\
    & \leq 8\tau^4; \label{eq:eq:hb_ode_shrinkage_variance_first_case}
  \end{align}

  \item for $x > 1$: it follows that
  \begin{align}
    \frac{1+x^2}{x^2}\left(1-e^{-\tau a}\left(\cos \tau b+\frac{a}{b}\sin \tau b\right)\right)^2
    &\leq\frac{1+x^2}{x^2}\left(1+(\tau x/\sqrt{\kappa}+1)e^{-\tau x/\sqrt{\kappa}}\right)^2 \notag \\
    &\leq 2\left(1+(\tau/\sqrt{\kappa}+1)e^{-\tau/\sqrt{\kappa}}\right)^2. \label{eq:eq:hb_ode_shrinkage_variance_second_case}
  \end{align}
\end{itemize}

We now let $\tau = \kappa^{1/6}$, i.e., $z = \kappa^{1/3}$.  Taking the maximum of \eqref{eq:eq:hb_ode_shrinkage_variance_first_case}, \eqref{eq:eq:hb_ode_shrinkage_variance_second_case} gives
\begin{align}
  \max\left\{8\tau^4, \, 2\left(1+(\tau/\sqrt{\kappa}+1)e^{-\tau/\sqrt{\kappa}}\right)^2 \right\} & = 8\tau^4 \notag \\
  & = 8\kappa^{2/3}, \label{eq:hb_vs_ridge_under_optimal_tuning_var_bound}
\end{align}
giving a bound on the variance term in \eqref{eq:heavy:predict}.

\paragraph{Putting it all together.}  Adding \eqref{eq:hb_vs_ridge_under_optimal_tuning_bias_bound}, \eqref{eq:hb_vs_ridge_under_optimal_tuning_var_bound} together and using $z = \kappa^{1/3}$ shows that we may take $C = h(\kappa)$ in \eqref{eq:heavy:predict}, completing the proof.  \qed

\subsection{Statement and proof of helper Lemma \ref{lem:hb_vs_ridge_under_optimal_tuning_helper_bias}}

\begin{lemma} \label{lem:hb_vs_ridge_under_optimal_tuning_helper_bias}
Let $s,t > 0$.  Define $a=\mu^{1/2}t$, $b=t(s-\mu)^{1/2}$, and $x=ts^{1/2}$.  Then, it holds that
\begin{align*}
  \left(\left(\cos b+\frac{a}{b}\sin b\right)e^{-a}\right)^2
  &\leq\left(\frac{x}{\sqrt{\kappa}}+1\right)^2e^{-2x/\sqrt{\kappa}}.
\end{align*}
\end{lemma}

\begin{proof}
Observe that $a\geq x/\sqrt{\kappa}$ and also that
\begin{equation}
\frac{a^2}{x^2-a^2}=\frac{a^2}{b^2}=\frac{\mu t^2}{(s-\mu)t^2}=\frac{\mu}{s-\mu}=\frac{1}{s/\mu-1}\geq\frac{1}{\kappa-1}.
\end{equation}
These together imply
\begin{align*}
\left(\left(\cos b+\frac{a}{b}\sin b\right)e^{-a}\right)^2
&=\left(\cos^2 b+\frac{2a}{b}\sin b\cos b+\frac{a^2\sin^2b}{b^2}\right)e^{-2a}\\
&\leq(1+2a+a^2)e^{-2a}\\
&\leq\left(\frac{x}{\sqrt{\kappa}}+1\right)^2e^{-2x/\sqrt{\kappa}}.
\end{align*}
The first line follows by simply expanding the square, the second by using $\sin x \leq x$ for $x \geq 0$, and the third by using $a\geq x/\sqrt{\kappa}$.  This completes the proof.
\end{proof}

\subsection{Statement and proof of helper Lemma \ref{lem:hb_vs_ridge_under_optimal_tuning_helper_var}}

\begin{lemma} \label{lem:hb_vs_ridge_under_optimal_tuning_helper_var}
Let $s,t > 0$.  Define $a=\mu^{1/2}t$, $b=t(s-\mu)^{1/2}$, and $x=ts^{1/2}$.  Then, it holds that
\begin{align*}
  & \left(1-\left(\cos b+\frac{a}{b}\sin b\right)e^{-a}\right)^2\leq4x^4, \quad x \leq 1; \\
  & \left(1-\left(\cos b+\frac{a}{b}\sin b\right)e^{-a}\right)^2 \leq \left(1+(x/\sqrt{\kappa}+1)e^{-x/\sqrt{\kappa}}\right)^2, \quad x > 1.
\end{align*}
\end{lemma}

\begin{proof}
To see the first inequality, we rewrite $1-(\cos b+a/b\cdot\sin b)\exp(-a)$ as
\begin{equation*}
\left(1-e^{-a}-a\right)+\left(a-a\frac{\sin b}{b}\right)+\left(e^{-a}\sin^2\frac{b}{2}\right)+\left(a(1-e^{-a})\frac{\sin b}{b}\right).
\end{equation*}
Now define two auxiliary functions $f_1(x)=x^2/2+1-x-e^{-x}$ and $f_2(x)=x^2/2-x+\sin x$. Since $f''_1(x)=1-e^{-x}\geq 0$ on $x\in[0,\infty)$, we have $f'_1(x)=x-1+e^{-x}\geq f'_1(0)=0$ on $x\in[0,\infty)$, which means $f_1(x)$ is non-decreasing, $f_1(x)\geq f_1(0)=0$. Thus, we have $|1-e^{-a}-a|\leq a^2/2\leq x^2/2$. Similarly, we can consider the derivatives of $f_2(x)$
\begin{equation*}
f_2'(x)=x-1+\cos x,~~~f''_2(x)=1-\sin x\geq 0.
\end{equation*}
So $f'_2(x)\geq f'_2(0)=0$, which means $f_2(x)\geq 0$ on $x\in[0,\infty)$, i.e. $x-\sin x\leq x^2/2$. Therefore, we have
\begin{equation*}
\left|a-a\frac{\sin b}{b}\right|=a\left|\frac{b-\sin b}{b}\right|\leq a\left|\frac{b^2/2}{b}\right|=\frac{ab}{2}\leq \frac{a^2+b^2}{4}=\frac{x^2}{4}.
\end{equation*}
The following bounds are straightforward:
\begin{equation*}
e^{-a}\sin^2\frac{b}{2}\leq\frac{x^2}{4},~~~\left|a(1-e^{-a})\frac{\sin b}{b}\right|\leq x^2,
\end{equation*}
where we used $|\sin x|\leq|x|$ and $1-\exp(-x)\leq x$ on $x\in[0,\infty)$. Putting all these together, we have 
\begin{equation*}
\left(1-\left(\cos b+\frac{a}{b}\sin b\right)e^{-a}\right)^2\leq\left(\frac{x^2}{2}+\frac{x^2}{4}+\frac{x^2}{4}+x^2\right)^2=4x^4,
\end{equation*}
as claimed.

To see the second inequality, following arguments similar to those used above shows that
\begin{equation*}
\left|\left(\cos b+\frac{a}{b}\sin b\right)e^{-a}\right|\leq (a+1)e^{-a}\leq(x/\sqrt{\kappa}+1)e^{-x/\sqrt{\kappa}},
\end{equation*}
which implies
\begin{equation*}
\left(1-\left(\cos b+\frac{a}{b}\sin b\right)e^{-a}\right)^2\leq\left(1+(x/\sqrt{\kappa}+1)e^{-x/\sqrt{\kappa}}\right)^2,
\end{equation*}
completing the proof.
\end{proof}

\subsection{Proof of Theorem \ref{thm:hb_ode_param_error}}
Fix $\varepsilon$.  Similar to what was done in the proof of Theorem \ref{thm:nest_ode_param_error}, putting the correspondence $\lambda = t^{-1/2}$ together with some basic manipulations, we obtain
\begin{equation*}
  \| \hat \beta^\hb(t) - \hat \beta^\ridge(1/t^2) \|_2^2 = \Big\| n^{-1/2} tV \Big( (t^2S)^+ (I - g^\hb(S,t)) - ( t^2S + I )^{-1} \Big) ((t^2S)^{1/2})^T U^T y \Big\|_2^2. 
\end{equation*}
So (again) we need to find a constant $C > 0$ satisfying
\begin{equation}
  \big( (t^2S)^+ (I - g^\hb(S,t)) - ( t^2S + I )^{-1} \big)^2 \preceq \big( C \cdot ( t^2S + I )^{-1} \big)^2. \label{eq:hb_ode_param_error_bound}
\end{equation}
Let
\begin{equation*}
  f(x)=\left(1-\left(\cos b+\frac{a}{b}\sin b\right)e^{-a}\right)\frac{x^2+1}{x^2}.
\end{equation*}
Plugging \eqref{eq:hb_ode_shrinkage} into \eqref{eq:hb_ode_param_error_bound}, we see that showing \eqref{eq:hb_ode_param_error_bound} is equivalent to showing $f(x) \leq C$, for $x \geq 0$ and some finite numerical constant $C > 0$.
Following arguments similar to those used when controlling the variance term in the proof of Lemma \ref{lem:hb_vs_ridge_under_optimal_tuning}, we see that $f(x)^2\leq16$ for $x \geq 0$, so that $(f(x)-1)^2\leq25$, completing the proof.  \qed

\subsection{Additional Numerical Examples}
Here we present some additional numerics for a broader range of feature covariance structures.  We consider sample covariance matrices $\hat \Sigma$ having i.i.d.~Gaussian and Student-t entries, as well as those satisfying the orthogonality condition $X^T X / n = s \cdot I$, for $s \in \{0.1, 1\}$.  In Figure \ref{fig:stability_gaus_student_t_ortho}, we plot the bias, variance, and risk of accelerated gradient flow \eqref{eq:nest_ode_risk}, heavy ball flow \eqref{eq:hb_ode_risk}, standard gradient flow (see Lemma 5 in \citet{AliKoTi19}), and ridge regression \eqref{eq:ridge_risk}, for each of these covariance structures and response points arising from the canonical model \eqref{eq:data_model}.  We follow the same generic experimental setup here as in Section \ref{sec:exps} of the main paper, i.e., we set $n = 500$, $p = 100$, and generate coefficients $\beta_0$ as well as set the noise variance $\sigma^2$ such that $\| \beta_0 \|_2^2/\sigma^2 = 1$.

We include the simulations with Gaussian and Student-t data to give a sense of how accelerated methods behave under the ``usual'' experimental conditions.  The results, plotted in the first two rows of Figure \ref{fig:stability_gaus_student_t_ortho}, roughly show most of the methods behaving similarly.

The story changes a bit when we adjust the scale of the eigenvalues of $\hat \Sigma$, as in the last two rows of the figure.  Here we see trends similar to those in Figure \ref{fig:stability_poly}, i.e., when $s = 0.1$ is (relatively) small, and $t$ is large enough, then we can see the accelerated bias drop faster than the standard gradient flow bias, and some instability arises.  When $s = 1$ is relatively large, the situation reverses.

\begin{figure*}[h!]
  \vspace{.3in}
  \centering{
    \includegraphics[scale=0.23]{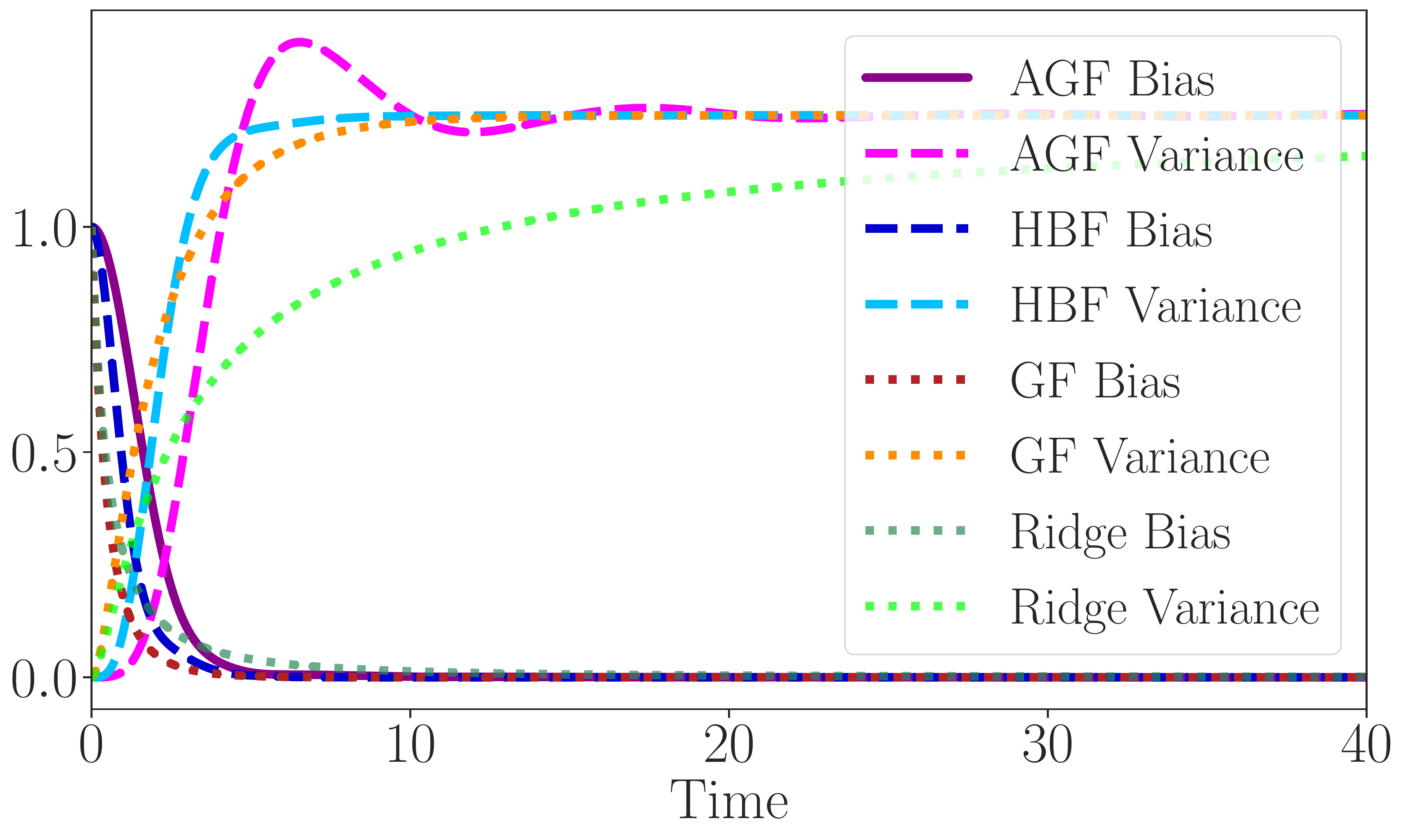} \hfill
    \includegraphics[scale=0.23]{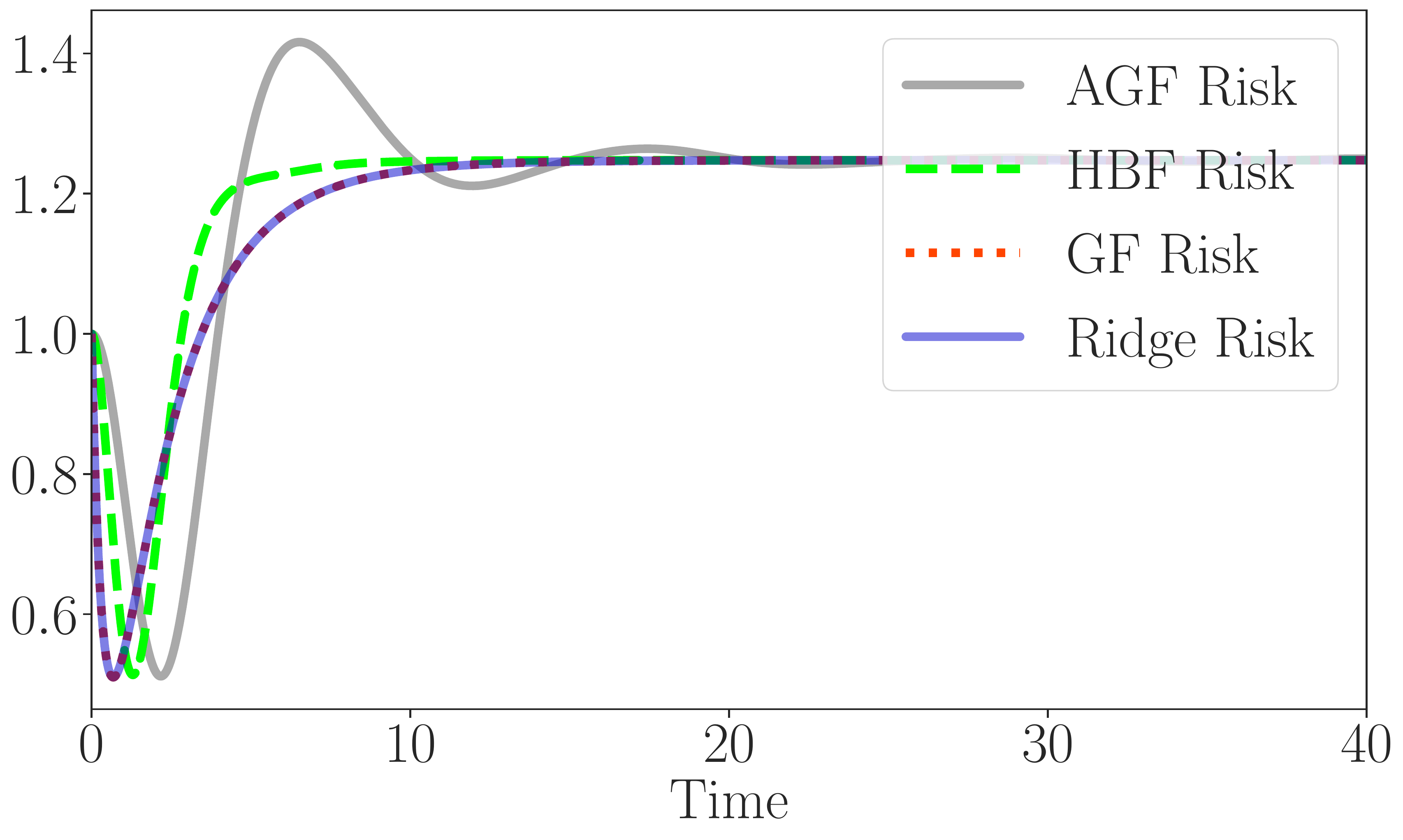} \\
    \includegraphics[scale=0.23]{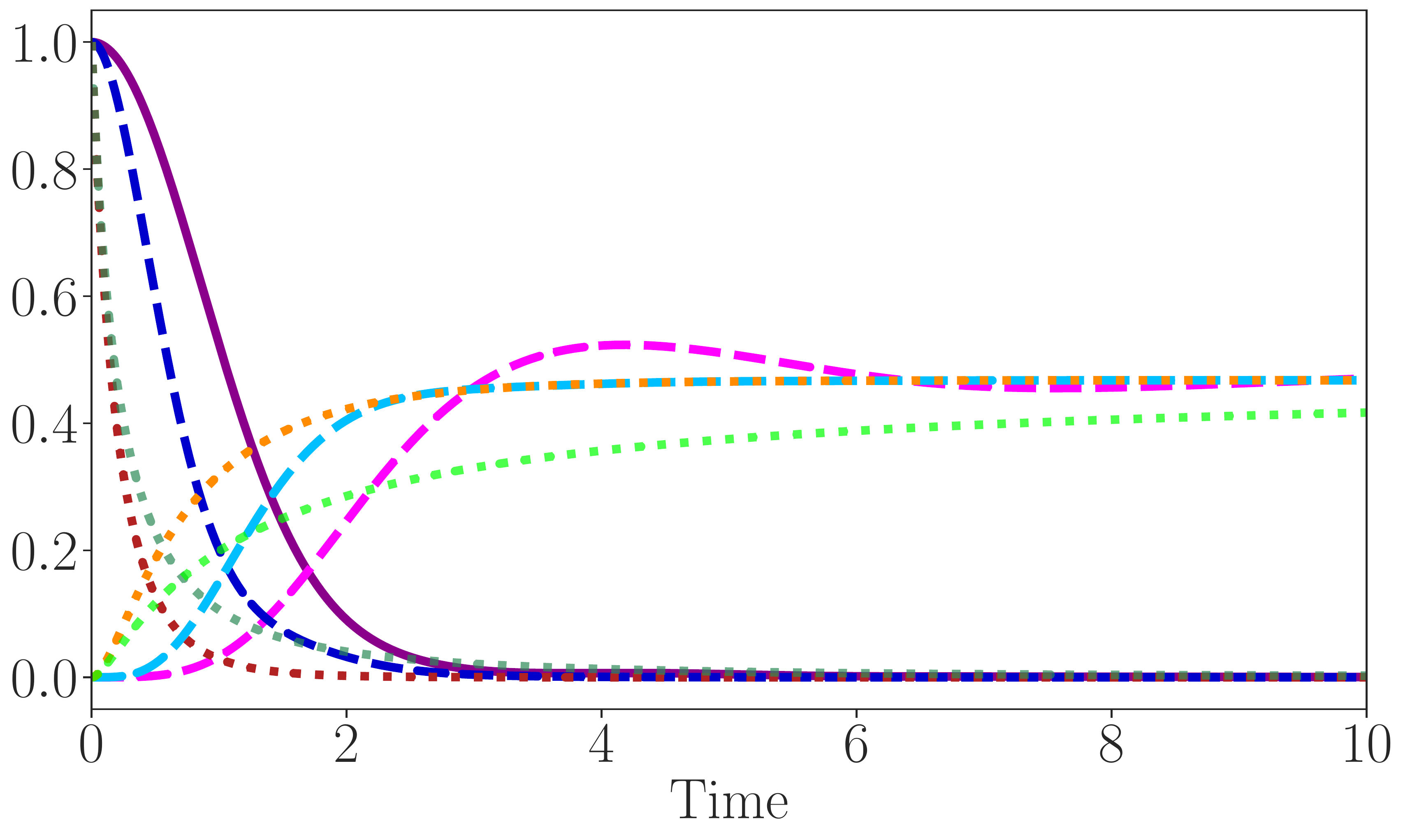} \hfill
    \includegraphics[scale=0.23]{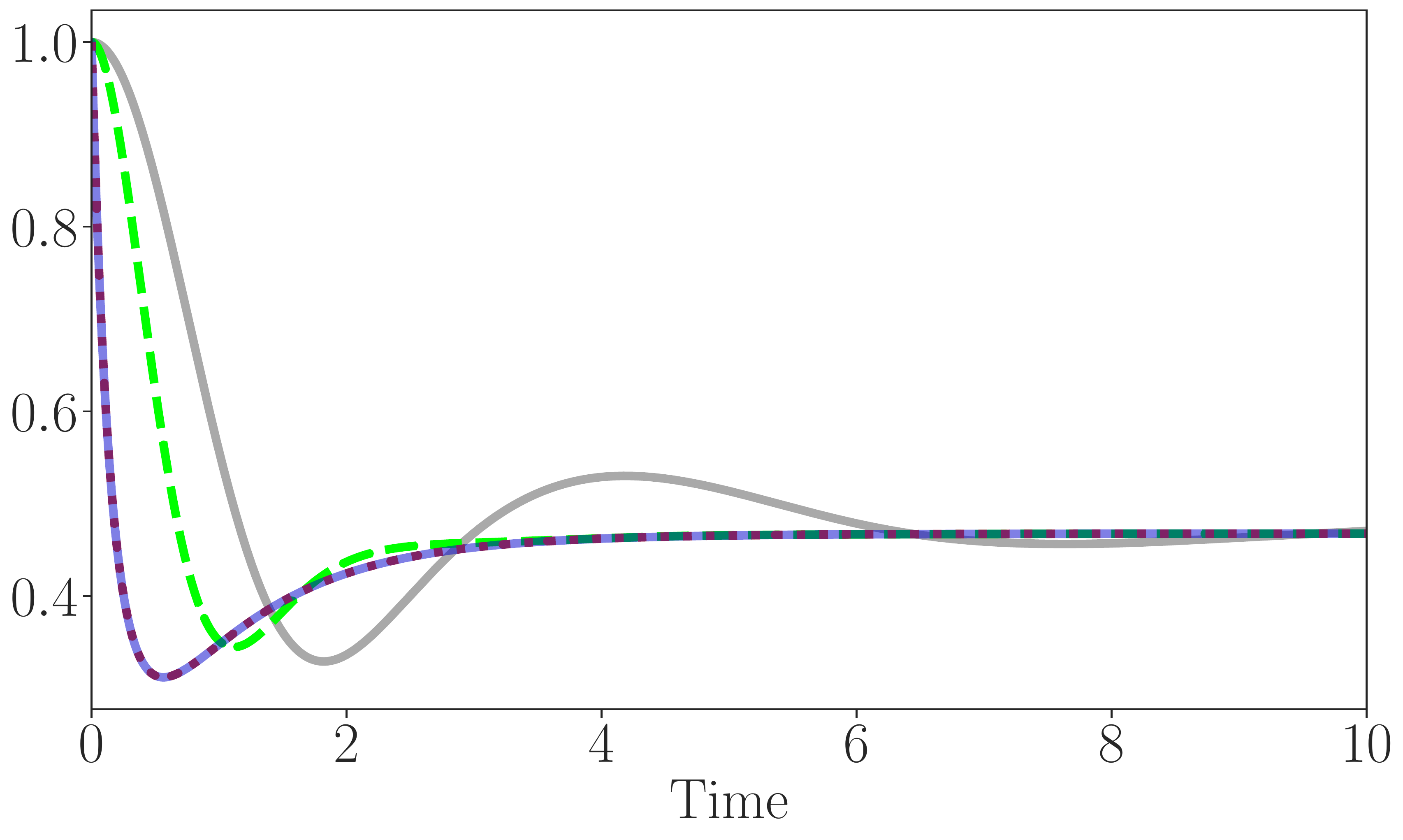} \\
    \includegraphics[scale=0.23]{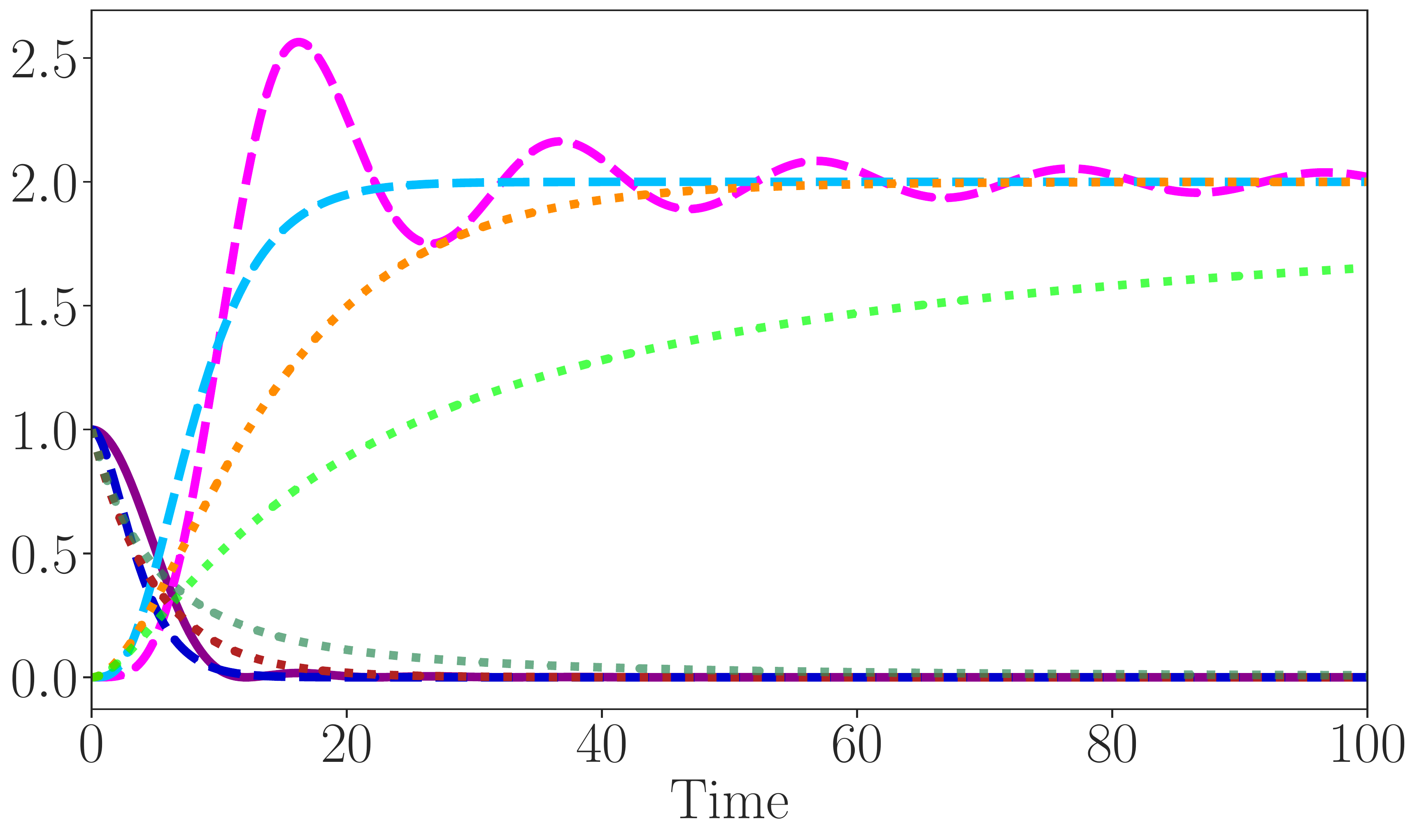} \hfill
    \includegraphics[scale=0.23]{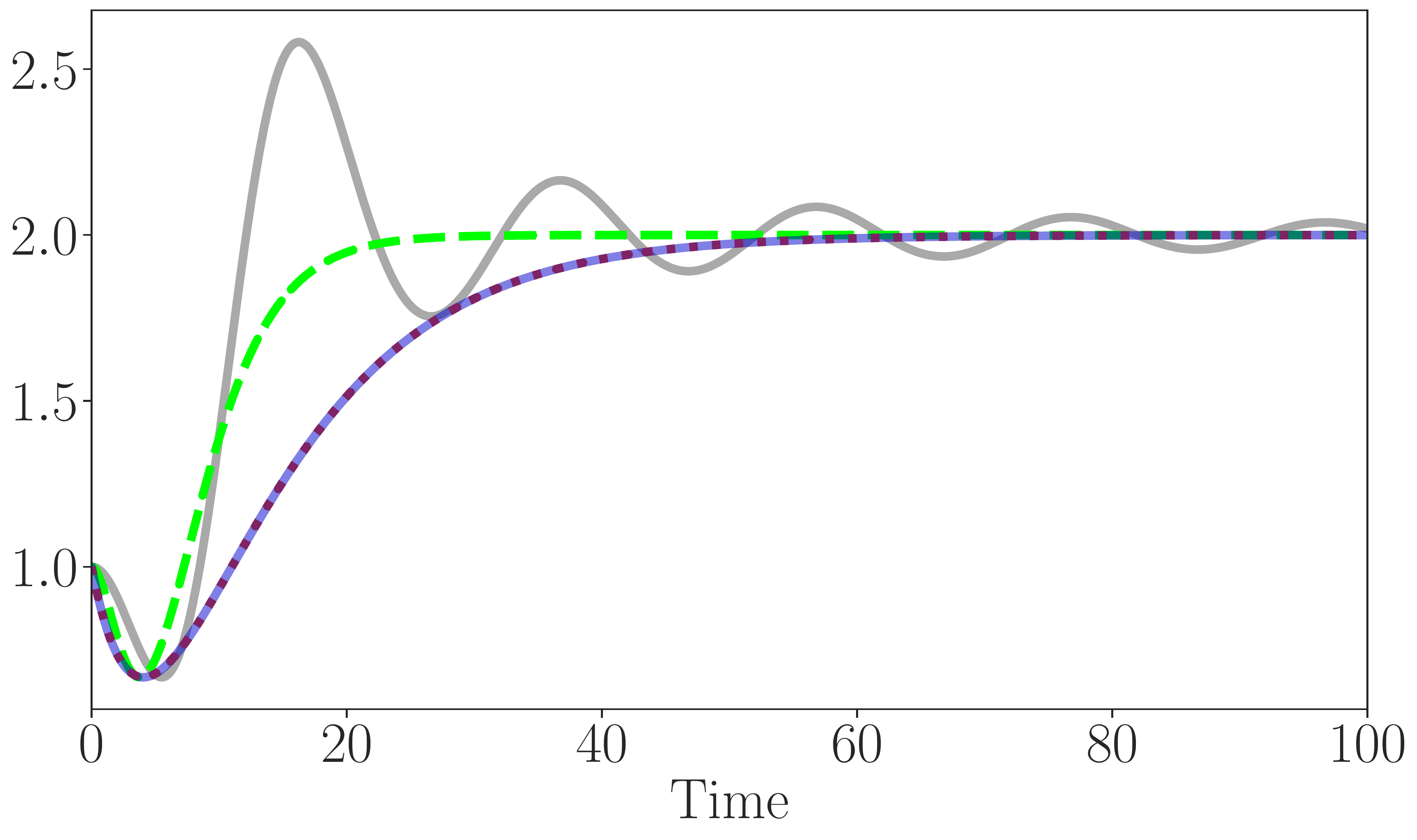} \\
    \includegraphics[scale=0.23]{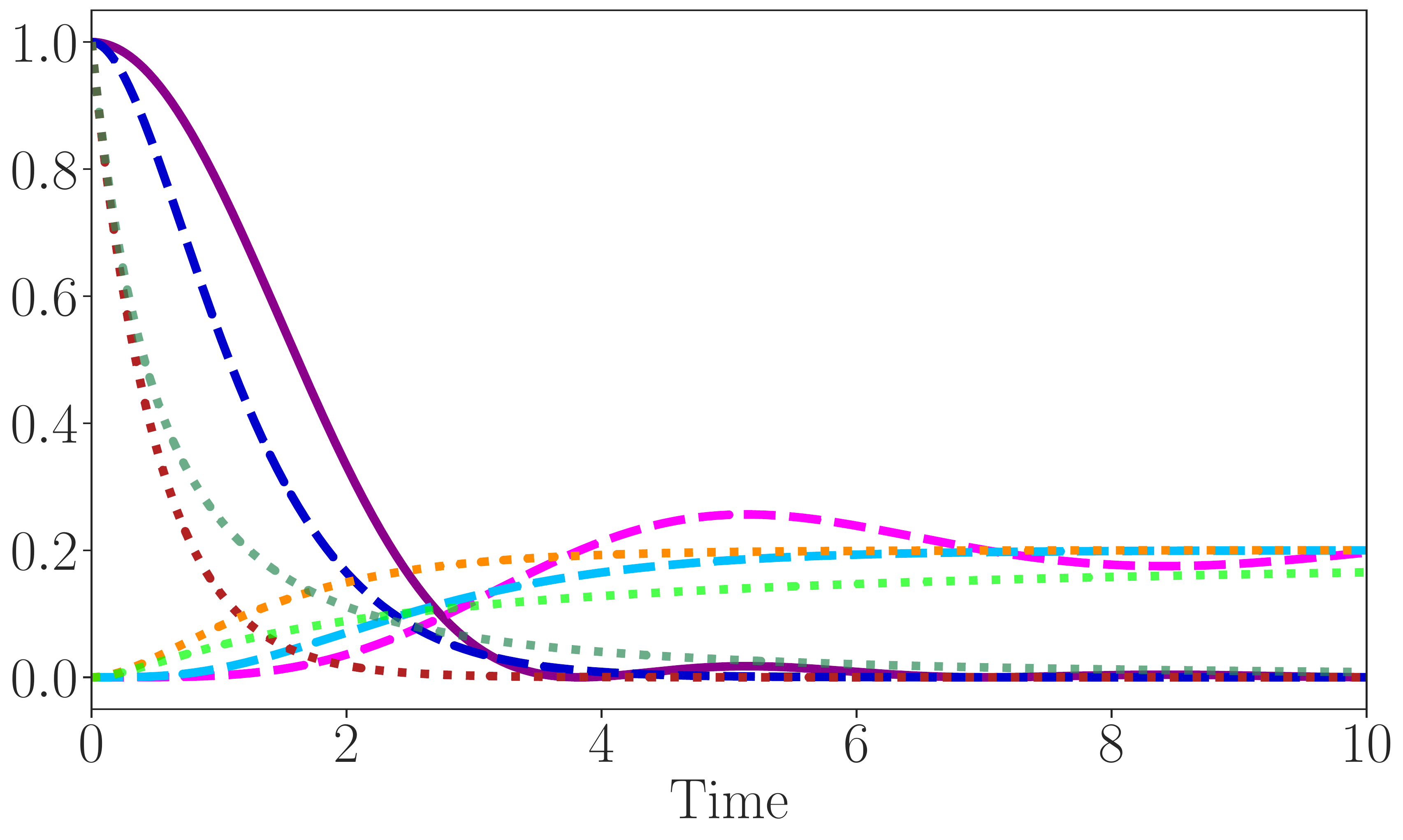} \hfill
    \includegraphics[scale=0.23]{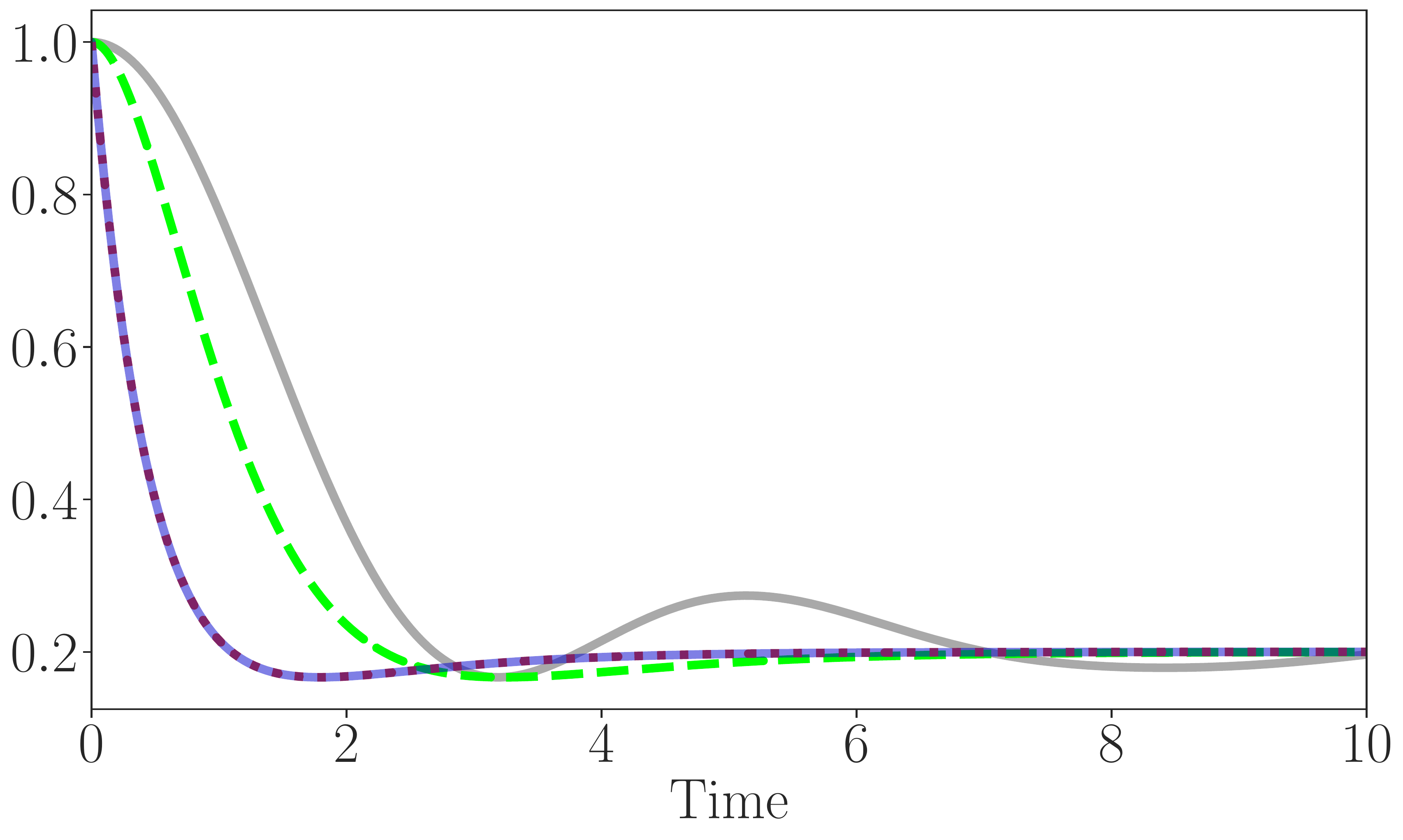} \\
  }
  \vspace{.3in}
  \caption{\textit{Bias, variance (left column), and risk (right column) for accelerated gradient flow \eqref{eq:nest_ode_risk}, heavy ball flow \eqref{eq:hb_ode_risk}, standard gradient flow (see Lemma 5 in \citet{AliKoTi19}), and ridge regression \eqref{eq:ridge_risk}.  The first and second row correspond to design matrices with i.i.d.~standard normal entries and Student-t entries, respectively.  The third and fourth rows correspond to scaled orthogonal design matrices, i.e., those satisfying $X^T X / n = s \cdot I$, for $s \in \{0.1, 1\}$, respectively.}}
  \label{fig:stability_gaus_student_t_ortho}
\end{figure*}

\bibliography{bib}
\bibliographystyle{plainnat}


\end{document}